\documentclass[10pt]{article}
\usepackage[margin=0.72in]{geometry}
\usepackage{amsmath,amssymb,amsthm,mathtools,booktabs,array,hyperref,microtype,enumitem}
\usepackage[T1]{fontenc}
\usepackage{lmodern}
\usepackage{xcolor}
\usepackage{graphicx}
\usepackage{tikz}
\usetikzlibrary{arrows.meta,positioning}
\usepackage{pgfplots}
\pgfplotsset{compat=1.18}
\definecolor{cohaccent}{HTML}{1F5FA8}
\definecolor{cohgray}{HTML}{9AA3AD}
\definecolor{cohred}{HTML}{C0392B}
\definecolor{cohgreen}{HTML}{2E8B57}
\hypersetup{
  colorlinks=true,
  linkcolor=black,
  citecolor=black,
  urlcolor=blue,
  pdftitle={Global Coherence: When Every Agent Is Right and the Team Is Still Wrong},
  pdfauthor={Xin Heng}
}
\setlist[itemize]{leftmargin=1.25em,itemsep=2pt,topsep=3pt}
\setlist[enumerate]{leftmargin=1.4em,itemsep=2pt,topsep=3pt}
\theoremstyle{definition}
\newtheorem{theorem}{Theorem}[section]
\newtheorem{proposition}[theorem]{Proposition}
\newtheorem{corollary}[theorem]{Corollary}
\newtheorem{lemma}[theorem]{Lemma}

\newcommand{\Match}{\mathrm{Match}}

\title{\textbf{Global Coherence: When Every Agent Is Right and the Team Is Still Wrong}\\
\large A Local-to-Global Semantic Foundation for Multi-Agent Collaboration}
\author{\textbf{Xin Heng}\\ Tote AI\\ \texttt{\small xin@tote.ai}}
\date{\today}

\begin{document}
\maketitle
\begin{abstract}
AI agents now share code, documents, tools, and analyses. No participant sees the whole state, so every step can look right while the combined result is wrong. We call this the \textbf{global coherence problem}. It is a problem of \emph{state}, not only of intelligence: a stronger model decides better locally, but it cannot see what is missing from its context. To our knowledge, this paper is the first to give this failure an exact condition and to turn it into a specification for \textbf{context engineering} and the \textbf{agent harness}.

The \textbf{Observation-Aliasing Impossibility Theorem} gives the condition. An agent can always act validly exactly when every world consistent with its context allows a common action. When look-alike worlds need different actions, no reasoning, role, message, or sample can guarantee the right one; with $k$ look-alike worlds whose valid actions are disjoint, the best randomized policy succeeds with probability $1/k$. The smallest sufficient context is NP-hard to find, yet refinement finds a sufficient one and keeps it current as the rules change. A second failure survives full information: valid pieces can fail to fit into one world. The same theorem, applied as state is split into views, translated, and changed, sets the conditions for coherence; sheaves, groupoids, and categories are one general way to meet them. \textbf{Models propose, and the harness decides what is valid.}

Ten studies test the theory. On TeamBench, teams overspend a shared budget in 5/5 runs; showing them the live count still leaves 4/5, and a harness that owns the budget leaves 0/5. Identity collisions that no pairwise check can see appear in 23 of 45 public ontology-alignment networks; agents with split views create them in 9 of 40 merges, and a network-level check removes all nine. A frontier LLM solves 40/40 revision tasks when it sees the deciding event. With the event hidden, no arm that reasons longer, votes, or splits into roles significantly beats chance, and a companion study replicates the result on models from a second provider. In $\tau^2$-bench Telecom, after a committed change is silently undone, agents with the harness score 1.00 against 0.07 with call-time policy checks. Refinement keeps a sufficient context at 28\% of the cost of sharing every field. Where a conventional mechanism already owns the state, it ties the harness, as predicted.
\end{abstract}
\section{Introduction: every step looks right, and the result is still wrong}\label{sec:intro}

Consider a three-agent team on a TeamBench task \cite{teambench2026}: a Planner writes the plan, an Executor runs commands, and a Verifier checks the result, with 20 command runs allowed for the whole team. Each agent spends runs for good reasons, and no single run is a mistake. But the limit belongs to the \emph{team}, no agent owns it, and the team spends about three times the budget. This is the failure this paper is about: \textbf{every step is locally justified, and together they break a rule that belongs to no one}. Showing the team the running total did not fix it. When the harness itself held the count and refused the 21st run, the violations disappeared; observed mean task progress did not fall in that five-run panel (Study~IV).

The pattern is general. In $\tau^2$-bench Telecom, a user and an assistant change different parts of one environment under different permissions \cite{tau2bench2025}; in BixBench3, several scientific artifacts must stay derived from one evolving analysis \cite{bixbench3}. The participants act on local views, each step can be locally correct, and the combined result need not be:
\[
\boxed{\text{local validity}\;\not\Rightarrow\;\text{global coherence}.}
\]

\paragraph{Two ways to fail.} Figure~\ref{fig:failure-modes} sketches both. The first is \textbf{observation aliasing}: two situations look the same to every participant but need different actions. The \textbf{Observation-Aliasing Impossibility Theorem} (Theorem~\ref{thm:aliasing}) says exactly when this is fatal: a policy can guarantee a valid action from what it observes---for an LLM agent, its context---if and only if every world consistent with that observation allows a common action. When that fails, no amount of reasoning, role decomposition, messaging, or sampling recovers the difference, and with $k$ look-alike worlds whose valid actions are disjoint the best randomized success is exactly $1/k$ (Theorem~\ref{thm:randombits}). A stronger model searches better within its context; it cannot see past it. The theorem separates \emph{intelligence}, what a model does with the information it has, from \emph{identifiability}, whether that information determines the right action at all. Context engineering helps exactly when it puts the missing distinction into the window. The second failure concerns information the participants \emph{do} have: pieces that pass every local check need not fit into any single global state. Three pairwise translations can each be inverse-consistent while the loop through all three does not close; three roles can each make a valid spend while the total breaks the budget.

\begin{figure}[b]
\centering
\begin{tikzpicture}[
  font=\small,
  world/.style={draw, rounded corners=2pt, fill=cohgray!35, minimum width=1.9cm, minimum height=0.8cm, align=center},
  view/.style={draw, rounded corners=2pt, fill=white, minimum width=1.6cm, minimum height=0.8cm, align=center},
  arr/.style={-{Stealth[length=2mm]}, thick},
]
\node[font=\bfseries] at (1.9,2.35) {(a) Missing information};
\node[world] (x1) at (0,1.2) {world $x$\\[-1pt]\scriptsize needs action $a$};
\node[world] (x2) at (0,-0.3) {world $x'$\\[-1pt]\scriptsize needs action $b$};
\node[view] (y) at (3.6,0.45) {same view\\[-1pt]$y=\Omega(x)=\Omega(x')$};
\draw[arr] (x1.east) -- (y.west);
\draw[arr] (x2.east) -- (y.west);
\node[align=center, font=\scriptsize, text width=4.8cm] at (1.9,-1.35)
  {Any policy acting on $y$ does the same thing in both worlds,\\so it is wrong in one of them.};

\begin{scope}[xshift=8.2cm]
\node[font=\bfseries] at (0.9,2.35) {(b) Pieces that do not glue};
\node[view] (A) at (-0.3,1.35) {view $A$};
\node[view] (B) at (2.5,1.35) {view $B$};
\node[view] (C) at (1.1,-0.35) {view $C$};
\draw[thick, cohgreen] (A) -- node[above=1pt, font=\scriptsize] {same \checkmark} (B);
\draw[thick, cohgreen] (B) -- node[right, font=\scriptsize, pos=0.45] {same \checkmark} (C);
\draw[thick, cohgreen] (C) -- node[left, font=\scriptsize, pos=0.55] {\checkmark\ opposite} (A);
\node[font=\Large, cohred] at (1.1,0.72) {$\circlearrowright$};
\node[align=center, font=\scriptsize, text width=5.0cm] at (1.1,-1.35)
  {Every pairwise check passes, but the loop does not close:\\no single world agrees with all three views.};
\end{scope}
\end{tikzpicture}
\caption{The two failure modes. (a) Two worlds that need different actions produce the same view, so no policy that acts on the view can be right in both (Theorem~\ref{thm:aliasing}). (b) Three views agree pairwise, yet no single world agrees with all three; the inconsistency lives on the loop, not on any edge (Sections~\ref{sec:realization} and~\ref{sec:representation}).}
\label{fig:failure-modes}
\end{figure}
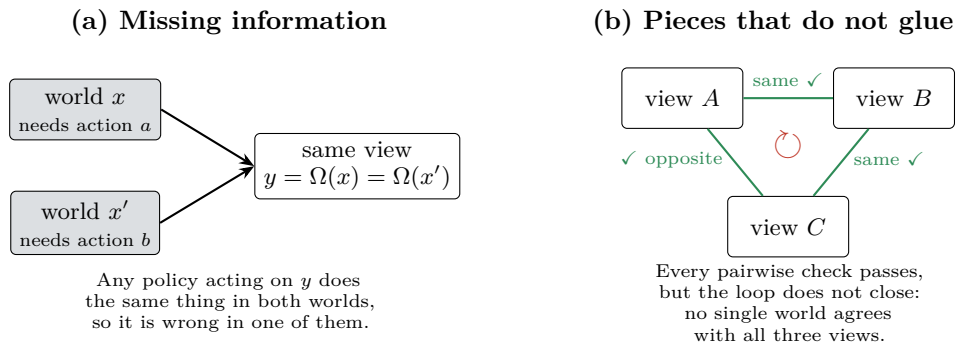

\paragraph{The framework.} The impossibility theorem says \emph{what} the harness must keep and every decision's context must reach: every distinction that changes which actions are valid (Proposition~\ref{prop:sound}). The rest of the paper answers \emph{how}. In a team that state is split into local views, re-encoded between participants, and changed by actions, and each step can silently lose a distinction the theorem says must be kept. Applying the theorem at each step gives the conditions for coherence, whatever mechanism meets them. Mathematics has a standard structure for each step; together they form
\[
\boxed{\mathfrak X=(H,\mathcal C,\mathcal G,\mathcal F;D).}
\]
Figure~\ref{fig:framework} shows all five on one example: Design, Fabrication, and Cost agents share one beam, each recording its length in its own unit.
\begin{itemize}
\item \textbf{A topology $H$: where things meet.} Which participants' scopes overlap, and on which state; here a triangle. On join trees pairwise agreement suffices; cycles can hide inconsistency from pairwise checks.
\item \textbf{A groupoid $\mathcal G$: what only looks changed.} The objects are encodings of one thing ($1$\,m, $1000$\,mm, $39.37$\,in) and the arrows are reversible translations between them. Translations composed around any loop must return the identity: a single wrong factor, $\div25$ instead of $\div25.4$, looks plausible on its edge and brings $1$\,m back as $1.016$\,m.
\item \textbf{A sheaf condition $\mathcal F$: whether one world glues.} Do local pieces that agree on every overlap come from exactly one world, here one beam? When this fails, the pieces agree everywhere and come from no world: every agent is right and the team is wrong.
\item \textbf{A category $\mathcal C$: what can change.} The arrows are actions that really change the state: revise the beam to $1.2$\,m, then order \$720 of steel. They compose into histories, most cannot be undone, and only they are checked at commit. $1\,\mathrm m\to1000\,\mathrm{mm}$ re-expresses the beam; $1\,\mathrm m\to1.2\,\mathrm m$ changes it.
\item \textbf{A history state $D$: what still matters.} Of the whole past, the harness keeps only what changes which future actions are valid: here the budget left and the version each agent last read, enough to refuse an order priced from the v3 quote once the beam is revised in v4.
\end{itemize}
Each structure keeps the distinctions the theorem requires intact through one step: splitting state into agents' views (the sheaf condition), translating it between formats (the groupoid), or changing it by actions (the category and history state). Each is unnecessary when its step is absent (Section~\ref{sec:bridge}, Corollary~\ref{cor:layers}). They are standard realizations, not the only possible ones: version checks, dependency engines, ledgers, escrow quotas, and residual solvers meet the same conditions in their own domains, and in our experiments each matches the general harness exactly where it already owns the relevant state. No category theory is needed to follow the argument.

\begin{figure}[tb]
\centering
\begin{tikzpicture}[
  font=\small,
  agent/.style={draw, rounded corners=2pt, fill=cohgray!25, minimum width=2.35cm, minimum height=0.85cm, align=center, inner sep=3pt},
  st/.style={draw, rounded corners=2pt, fill=white, text width=1.75cm, minimum height=0.85cm, align=center, inner sep=3pt},
  tag/.style={circle, fill=cohaccent, text=white, font=\scriptsize\bfseries, inner sep=0.8pt, minimum size=0.46cm},
  arr/.style={-{Stealth[length=2mm]}, thick},
]
\node[font=\bfseries] at (0,2.6) {(a) One beam, three agents, three units};
\node[agent] (d) at (-1.95,1.35) {Design agent\\[-1pt]\scriptsize beam $=1$ m};
\node[agent] (f) at (1.95,1.35) {Fabrication agent\\[-1pt]\scriptsize beam $=1000$ mm};
\node[agent] (c) at (0,-1.1) {Cost agent\\[-1pt]\scriptsize beam $=39.37$ in};
\draw[arr] (d) -- node[above, font=\scriptsize] {$\times1000$} (f);
\draw[arr] (f) -- node[right, font=\scriptsize, xshift=2pt] {$\div25.4$} (c);
\draw[arr] (c) -- node[left, font=\scriptsize, xshift=-2pt] {$\times0.0254$} (d);
\node[font=\Large, cohaccent] at (0,0.62) {$\circlearrowright$};
\node[font=\scriptsize, align=center] at (0,0.12) {back to\\[-2pt]1 m?};
\node[tag] at (-3.35,0.15) {$H$};
\node[tag] at (0,2.02) {$\mathcal G$};
\node[tag] at (0.62,0.62) {$\mathcal F$};

\begin{scope}[xshift=4.95cm]
\node[font=\bfseries] at (3.35,2.6) {(b) Changes, commits, and memory};
\node[st] (s0) at (0,1.35) {v3\\[-1pt]\scriptsize beam 1 m\\[-2pt]\$1000 left};
\node[st] (s1) at (3.35,1.35) {v4\\[-1pt]\scriptsize beam 1.2 m\\[-2pt]\$1000 left};
\node[st] (s2) at (6.7,1.35) {v5\\[-1pt]\scriptsize beam 1.2 m\\[-2pt]\$280 left};
\draw[arr] (s0) -- node[above, font=\scriptsize] {revise} node[below, font=\scriptsize] {to 1.2 m} (s1);
\draw[arr] (s1) -- node[above, font=\scriptsize] {order} node[below, font=\scriptsize] {\$720} (s2);
\node[tag] at (1.675,2.08) {$\mathcal C$};
\node[draw=cohaccent, fill=cohaccent!10, rounded corners=2pt, text width=7.6cm, align=left, font=\scriptsize, inner sep=4pt] (mem) at (3.35,-0.7)
  {\textbf{Runtime memory:} the budget left, and the version each agent last read. Nothing else in the log is needed. An order priced from the Cost agent's v3 quote is refused at commit in v4: it was priced for a beam that no longer exists.};
\node[tag] at ([xshift=-2pt, yshift=2pt]mem.north west) {$D$};
\end{scope}
\end{tikzpicture}
\caption{The framework $\mathfrak X=(H,\mathcal C,\mathcal G,\mathcal F;D)$ on one example. (a) Three agents share one beam. $H$ is the triangle of who shares what, and because it is a cycle, pairwise checks are not enough. $\mathcal G$ is the unit translations on its edges, which must return the identity around the loop. $\mathcal F$ asks whether the three local values come from one beam. (b) $\mathcal C$ is the real changes, which compose into a history and are checked at commit. $D$ is the little the harness must remember to check them; here it refuses an order priced for the old beam.}
\label{fig:framework}
\end{figure}

\paragraph{Why sharing everything is not the answer.} Context windows, bandwidth, privacy, and authority limit what each participant may see, so the real question is what to share. Finding the smallest sufficient context is NP-hard (Theorem~\ref{thm:instrument}), but refinement finds a sufficient one and recertifies it when the rules change (Proposition~\ref{prop:refine}). Complete information does not remove the second failure. Pairwise checks can all pass while the whole is inconsistent (Theorem~\ref{thm:bfmy}, Proposition~\ref{prop:nok}). Spends that are each valid can break a shared total (Proposition~\ref{prop:inconfluent}). A commit is safe only if what it read is still current (Proposition~\ref{prop:occ}). Detecting an inconsistency is cheap; choosing the smallest repair is NP-hard (Corollary~\ref{cor:span}, Proposition~\ref{prop:z2repair}).

\paragraph{What this means for building agent systems.} The theory separates two jobs. The model supplies \emph{proposal intelligence}: interpreting intent, planning, choosing tools, drafting changes. The harness supplies \emph{state, validity, and commit}: it decides what enters each agent's context, holds the distinctions that matter, and accepts, repairs, or rejects each proposed tool call. Stronger models improve the first job; they cannot replace the second when the needed distinction is missing. The limits bind every layer. No matter how capable the model or how well engineered its context, a constraint that several agents act on needs an owner (Proposition~\ref{prop:inconfluent}, Study~IV). And no matter how carefully the harness is built, it guarantees only what its own state can tell apart, catches a contradiction around a loop only if some check spans the loop, and catches a stale read only if it remembers what each proposal read (Theorem~\ref{thm:aliasing} applied to the harness itself, Study~V). A manager or orchestrator agent is no substitute. It is a model that proposes from its own view, bound by the same theorem (Corollary~\ref{cor:roles}); in Study~IV a team with a dedicated Planner and Verifier still overspent when shown the live count. Where an orchestrator does work, the framework says why: a hub that holds the shared fields makes the overlaps a tree, and on trees local checks suffice. The same results answer team-design questions usually settled by intuition: who owns a decision, who can work in parallel, and who must wait (Section~\ref{sec:team-design}).

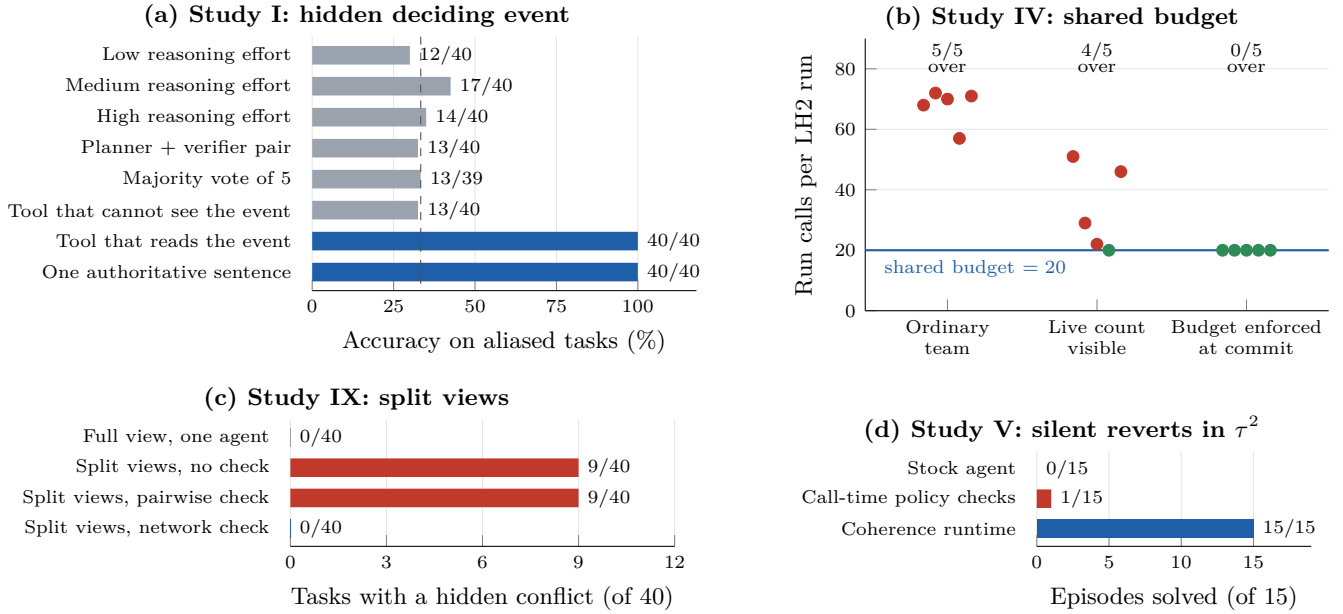
\begin{figure}[t]
\centering
\begin{minipage}[b]{0.54\linewidth}\centering
\textbf{\small (a) Study I: hidden deciding event}\\[2pt]
\begin{tikzpicture}
\begin{axis}[
  scale only axis, width=\dimexpr\linewidth-4.6cm\relax,
  height=3.36cm, xbar, bar width=7pt, bar shift=0pt,
  xmin=0, xmax=118, xlabel={Accuracy on aliased tasks (\%)}, y dir=reverse,
  ytick={0,...,7}, yticklabels={{Low reasoning effort},{Medium reasoning effort},{High reasoning effort},{Planner + verifier pair},{Majority vote of 5},{Tool that cannot see the event},{Tool that reads the event},{One authoritative sentence}},
  ymin=-0.6, ymax=7.6, enlarge y limits=false,
  yticklabel style={font=\scriptsize}, xticklabel style={font=\scriptsize},
  xlabel style={font=\small}, axis x line*=bottom, axis y line*=left,
  xmajorgrids, grid style={black!10}, y axis line style={draw=none}, ytick style={draw=none},
  nodes near coords, point meta=explicit symbolic,
  every node near coord/.append style={font=\scriptsize, anchor=west},
  xtick={0,25,50,75,100},
]
\addplot[fill=cohgray, draw=none] coordinates {(30.0,0) [12/40] (42.5,1) [17/40] (35.0,2) [14/40] (32.5,3) [13/40] (33.3,4) [13/39] (32.5,5) [13/40]};
\addplot[fill=cohaccent, draw=none] coordinates {(100.0,6) [40/40] (100.0,7) [40/40]};
\draw[dashed, black!70] (axis cs:33.33,-0.6) -- (axis cs:33.33,7.6) node[pos=0, anchor=south, font=\scriptsize, inner sep=1pt] {chance $=1/3$};
\end{axis}
\end{tikzpicture}
\end{minipage}\hfill
\begin{minipage}[b]{0.42\linewidth}\centering
\textbf{\small (b) Study IV: shared budget}\\[2pt]
\begin{tikzpicture}
\begin{axis}[
  scale only axis, width=\dimexpr\linewidth-1.4cm\relax, height=3.6cm, ymin=0, ymax=90,
  ylabel={Run calls per LH2 run}, ylabel style={font=\small},
  xtick={0,1,2}, xticklabels={{Ordinary\\team},{Live count\\visible},{Budget enforced\\at commit}}, xmin=-0.55, xmax=2.55,
  xticklabel style={align=center, font=\scriptsize}, yticklabel style={font=\scriptsize},
  axis x line*=bottom, axis y line*=left, ymajorgrids, grid style={black!10},
]
\draw[cohaccent, thick] (axis cs:-0.55,20) -- (axis cs:2.55,20) node[pos=0.02, below right, font=\scriptsize, cohaccent] {shared budget $=20$};
\addplot[only marks, mark=*, mark size=2.2pt, color=cohred] coordinates {(-0.16,68) (-0.08,72) (0.00,70) (0.08,57) (0.16,71) (0.84,51) (0.92,29) (1.00,22) (1.16,46)};
\addplot[only marks, mark=*, mark size=2.2pt, color=cohgreen] coordinates {(1.08,20) (1.84,20) (1.92,20) (2.00,20) (2.08,20) (2.16,20)};
\node[font=\scriptsize, anchor=south, align=center] at (axis cs:0,76) {5/5\\[-2pt]over};
\node[font=\scriptsize, anchor=south, align=center] at (axis cs:1,76) {4/5\\[-2pt]over};
\node[font=\scriptsize, anchor=south, align=center] at (axis cs:2,76) {0/5\\[-2pt]over};
\end{axis}
\end{tikzpicture}
\end{minipage}\\[8pt]
\begin{minipage}[b]{0.54\linewidth}\centering
\textbf{\small (c) Study IX: split views}\\[2pt]
\begin{tikzpicture}
\begin{axis}[
  scale only axis, width=\dimexpr\linewidth-4.6cm\relax,
  height=1.68cm, xbar, bar width=7pt, bar shift=0pt,
  xmin=0, xmax=12, xlabel={Tasks with a hidden conflict (of 40)}, y dir=reverse,
  ytick={0,...,3}, yticklabels={{Full view, one agent},{Split views, no check},{Split views, pairwise check},{Split views, network check}},
  ymin=-0.6, ymax=3.6, enlarge y limits=false,
  yticklabel style={font=\scriptsize}, xticklabel style={font=\scriptsize},
  xlabel style={font=\small}, axis x line*=bottom, axis y line*=left,
  xmajorgrids, grid style={black!10}, y axis line style={draw=none}, ytick style={draw=none},
  nodes near coords, point meta=explicit symbolic,
  every node near coord/.append style={font=\scriptsize, anchor=west},
  xtick={0,3,6,9,12},
]
\addplot[fill=cohgray, draw=none] coordinates {(0,0) [0/40]};
\addplot[fill=cohred, draw=none] coordinates {(9,1) [9/40] (9,2) [9/40]};
\addplot[fill=cohaccent, draw=none] coordinates {(0,3) [0/40]};
\end{axis}
\end{tikzpicture}
\end{minipage}\hfill
\begin{minipage}[b]{0.42\linewidth}\centering
\textbf{\small (d) Study V: silent reverts in $\tau^2$}\\[2pt]
\begin{tikzpicture}
\begin{axis}[
  scale only axis, width=\dimexpr\linewidth-3.9cm\relax,
  height=1.26cm, xbar, bar width=7pt, bar shift=0pt,
  xmin=0, xmax=19, xlabel={Episodes solved (of 15)}, y dir=reverse,
  ytick={0,...,2}, yticklabels={{Stock agent},{Call-time policy checks},{Coherence runtime}},
  ymin=-0.6, ymax=2.6, enlarge y limits=false,
  yticklabel style={font=\scriptsize}, xticklabel style={font=\scriptsize},
  xlabel style={font=\small}, axis x line*=bottom, axis y line*=left,
  xmajorgrids, grid style={black!10}, y axis line style={draw=none}, ytick style={draw=none},
  nodes near coords, point meta=explicit symbolic,
  every node near coord/.append style={font=\scriptsize, anchor=west},
  xtick={0,5,10,15},
]
\addplot[fill=cohred, draw=none] coordinates {(0,0) [0/15] (1,1) [1/15]};
\addplot[fill=cohaccent, draw=none] coordinates {(15,2) [15/15]};
\end{axis}
\end{tikzpicture}
\end{minipage}
\caption{Headline results, all with the same LLM (models and settings in Section~\ref{sec:empirical}). (a) The LLM solves all 40 tasks when the deciding event is visible. With it hidden, the tested arms score 12--17 of 40 and none is significantly above chance ($1/3$); adding the missing fact (blue) gives 40/40 again. One majority-vote cell failed on an API error, so that arm has 39 tasks. (b) Each dot is one of five runs of TeamBench's budgeted-workflow task (LH2). Ordinary teams overspend the 20-call budget in 5/5 runs; showing the live count still leaves 4/5; enforcing the budget at commit stops every run at exactly 20. (c) Ontology-alignment tasks under the stated unique-name rule: agents with split views leave hidden collisions in 9 of 40 tasks; a pairwise check never fires; a network-level check removes all nine, with the repair-quality cost reported in Study~IX. (d) $\tau^2$-bench Telecom tasks in which a change the agent already committed is quietly undone, scored by $\tau^2$'s own grader: the stock agent and call-time policy checks believe the job is done; the coherence harness sends one notice and the agent redoes the change in 15 of 15 episodes. Bars and dots are plotted directly from the recorded run results.}
\label{fig:headline}
\end{figure}

\paragraph{Four fixes that do not work.} The experiments target the parts of the theory that run against intuition.
\begin{itemize}
\item \textbf{A stronger model is not the fix for missing information.} The same model solves every task when it can see the deciding fact, scores at chance when the fact is hidden, however it reasons, deliberates, or votes, and solves every task again when one sentence states the fact (Study~I).
\item \textbf{More context is not the fix for shared state.} Agents with the live budget in their context still overspend it. A shared constraint needs an owner: enforcement at commit, or quotas that split it (Study~IV).
\item \textbf{Checking every handoff is not the fix.} Every pairwise review can pass while the integrated result is impossible, in public benchmark data and in agent runs alike (Studies~VIII and~IX).
\item \textbf{Checking the current state is not the fix.} A commit can rest on a stale read, and a committed change can be silently undone, while every check on the current state passes (Study~V).
\end{itemize}

\paragraph{What is new.} Failure taxonomies of multi-agent LLM systems list inter-agent misalignment among the most common failures \cite{cemri2025mast}. To our knowledge this is the first work to give such failures an exact condition, to separate identifiability and commit semantics from model capability, and to turn the limits into specifications for context and harness. Acyclic database schemes, group synchronization, sheaves, treewidth, Myhill--Nerode minimization, abstraction refinement, and optimistic concurrency control are each established, and we credit each where it appears. The new formal results form the bridge to agents: the validity form of the aliasing theorem with its exact randomized value and bit bounds, the hardness of choosing the smallest sufficient state, certified refinement, and the result that each structure is the theorem applied at one map (Appendix~\ref{app:provenance} sorts every result by origin). Across the studies, the predicted failures occur, the tested inference variations do not remove them, and owning the missing state does (Figure~\ref{fig:headline}).\footnote{The code and run records for all ten studies will be released.}

\paragraph{Contributions.}
\begin{itemize}
\item \textbf{Sufficiency: what a team must know, how to monitor it, and what it costs.} The Observation-Aliasing Impossibility Theorem with its exact randomized value, bit bounds, and multi-agent and tool-use extensions: more reasoning, agents, messages, and samples cannot recover a distinction none of them can see; adding it to context can. The smallest sufficient context is NP-hard to find, but a sufficient one is not: counterexample-guided refinement finds it, certifies it, recertifies it when the rules change, and raises an alert, never a pass, when no available field restores sufficiency. A harness can monitor what it cannot optimize.
\item \textbf{A bridge from sufficiency to coherence.} The same theorem, applied as state is split into views, translated between formats, and changed by actions, gives the conditions for coherence: views that separate and glue, translations that keep distinctions and close around loops, and a state that updates itself. Sufficiency and coherence are separate obligations, and each structure is needed exactly when its step is present. The conditions fix no mechanism. They say when contracts, version checks, ledgers, and escrow already suffice, why a constraint that spans agents needs an owner, and why a commit must recheck everything the agent read.
\item \textbf{Coherence mechanisms and their costs.} Topology, sheaves, groupoids, and categories give one general way to meet the conditions. Pairwise checks suffice exactly on join trees, translation loops can be certified in linear time, and finding a conflict is cheap while the smallest repair is NP-hard. Familiar tools are special cases, which explains where each is enough.
\item \textbf{A specification for agent harnesses.} Models propose; a harness that owns the state decides what is valid. When a multi-agent system fails, first ask which distinction no one owns, before asking for a bigger model or more agents.
\item \textbf{Evidence on both sides of the boundary.} Ten studies, four of them causal agent experiments frozen before scoring. The harness removes the targeted failures without a stronger model, and it ties conventional tools where they already own the state.
\end{itemize}

The paper follows four questions. Do the views contain enough information to act validly (Section~\ref{sec:model})? Do locally compatible pieces form one global state (Section~\ref{sec:coherence-math})? Do actions keep it that way, and what must the harness remember and check (Section~\ref{sec:architecture})? And do the predicted failures and fixes appear in practice (Section~\ref{sec:empirical})?
\section{Model and the limit of reasoning}\label{sec:model}

This section states the result everything else rests on: no amount of reasoning recovers a distinction that is not in view. Proofs and quantitative refinements are in Appendix~\ref{app:aliasing}.

\subsection{What participants see and what they may do}\label{sec:formal}

Let $X$ be the set of possible global states and $\mathcal A$ the set of actions. For each state $x$, let $C(x)\subseteq\mathcal A$ be the valid actions there. $C$ can encode anything that decides validity: authority, budgets, invariants, freshness, or required artifacts. A participant does not see $x$ itself; it sees an observation $y=\Omega(x)$. In a team, $\Omega(x)$ can be the tuple of everyone's local views $(\Omega_1(x),\ldots,\Omega_m(x))$. A runtime remembers some summary $\phi(x)$ of the state. In agent terms, $\Omega_i$ is agent $i$'s context (prompt, retrieved documents, tool results, memory), and the runtime is the agent harness, the code around the model that assembles that context, executes tool calls, and decides which take effect; we use the two words interchangeably.

For the local-to-global part, split a scope $U$ into local scopes $U_1,\ldots,U_n$: participants' views, subsystems, or artifacts. Write $X(U)$ for global states and $\Match(\{U_i\})$ for tuples of local states that agree wherever scopes overlap. Restricting a global state to its local pieces gives a map
\[
\rho:X(U)\rightarrow\Match(\{U_i\}).
\]
Both failures from the introduction are properties of this one map. In \textbf{aliasing}, $\rho$ is not one-to-one: two different worlds give the same local views. In \textbf{non-realizability}, $\rho$ is not onto: locally matching pieces come from no world at all. A coherent system needs $\rho$ to be a bijection: every world gives distinct views, and every matching set of views comes from a world. That is exactly the sheaf condition; when equivalent representations are tracked as well, it becomes the descent condition \cite{stacks00VG,stacks02ZH,stacks042Y}. This section is about the first failure; Section~\ref{sec:coherence-math} develops the second.

\par\medskip\noindent\begin{minipage}{\linewidth}
\textbf{Notation at a glance.}\par\smallskip
\centering
\small
\begin{tabular}{>{\raggedright\arraybackslash}p{2.6cm}>{\raggedright\arraybackslash}p{12.4cm}}
\toprule
Symbol & Meaning\\
\midrule
$x\in X$ & a possible state of the world (a \emph{world})\\
$C(x)\subseteq\mathcal A$ & the actions that are \emph{valid} in world $x$\\
$y=\Omega(x)$ & what the team observes; agent $i$'s part $\Omega_i(x)$ is its context\\
$\Omega^{-1}(y)$ & the \emph{fiber} of $y$: every world that produces the same observation $y$\\
$z=\phi(x)$ & the summary of the world that the harness remembers\\
$\mathcal B(z)$ & the actions \emph{guaranteed valid} at $z$: valid in every world with summary $z$ (Section~\ref{sec:exact-state})\\
$\rho$ & the map from a global state to the agents' local views\\
$H,\mathcal G,\mathcal F,\mathcal C,D$ & who overlaps with whom; translations between formats; the gluing check; real state changes; the history kept (Section~\ref{sec:intro})\\
\bottomrule
\end{tabular}
\normalsize
\end{minipage}\par\medskip

\subsection{The impossibility theorem}\label{sec:observation}

The central result makes a simple point exact: no policy can act on a distinction it cannot see. The idea has a long history, as perceptual aliasing in reinforcement learning \cite{whiteheadballard1991} and as indistinguishability arguments in distributed computing \cite{flp1985}. What we add is the exact condition for choosing a valid action, and what it implies for teams, randomness, tools, and hints. Call the set of worlds that produce the same observation $y$ the \emph{fiber} of $y$.

\begin{theorem}[Observation-Aliasing Impossibility Theorem]\label{thm:aliasing}
\emph{In one sentence: a policy can always pick a valid action exactly when, for every observation, all the worlds that could be behind it share at least one valid action.}

Precisely, let $\Omega:X\to Y$ be what the decision maker observes and $C(x)$ the valid actions in world $x$. A deterministic policy $\pi:Y\to\mathcal A$ that is always valid,
\[
\pi(\Omega(x))\in C(x)\qquad\text{for every world }x,
\]
exists if and only if every fiber has a common valid action:
\[
\bigcap_{x\in\Omega^{-1}(y)}C(x)\neq\varnothing\qquad\text{for every observation }y.
\]
\end{theorem}

\emph{Proof sketch.} If an always-valid policy exists, its action $\pi(y)$ must be valid for every state in the fiber $\Omega^{-1}(y)$, hence lies in the intersection. Conversely, choose one action from each nonempty intersection. Appendix~\ref{app:aliasing} gives the full proof.

The typical case has two worlds that look identical, $\Omega(x)=\Omega(x')$, but allow disjoint actions, $C(x)\cap C(x')=\varnothing$; no policy acting on the observation can be right in both. This holds for any decision process whose complete information is $\Omega(x)$: one agent, a team, a human organization, or a human--AI collaboration. In Study~I, three teams commit, revert, and re-apply changes to one shared parameter; if the last event is hidden, three histories produce the same log but leave three different heads, and whatever head an agent names is wrong in two of the three worlds. Adding agents does not change this when they only talk among themselves:

\begin{corollary}[Reasoning and role decomposition do not create missing state information]\label{cor:roles}
\emph{In one sentence: agents that only talk among themselves cannot learn anything their combined observations did not already contain.}

Precisely, consider any deterministic protocol in which agents start with views $\Omega_1(x),\ldots,\Omega_m(x)$ and then exchange messages computed only from those views and earlier messages, with no new observation of the world. Then the whole transcript and the final action are functions of the joint observation $\Omega(x)$ alone. Worlds that the joint observation cannot tell apart stay indistinguishable, however many agents, roles, reasoning rounds, or messages are added.
\end{corollary}

The proof is an induction over communication rounds. Randomness changes which answer is chosen, not which worlds can be told apart. Conversely, when a team fails on a distinction its joint observation \emph{does} contain, the failure is one of protocol, not an impossibility. And each action is limited by what the acting role has seen and been told, not by the union of everyone's views. In Study~I, more reasoning effort, reviewer roles, and majority votes all left the hidden-event tasks at chance, while one sentence stating the event solved 40 of 40.

\paragraph{How much can randomness, tools, and hints buy?} When no policy can guarantee a valid action, two questions remain: how often can a policy still be right, and what would it take to do better?

\begin{center}
\small
\begin{tabular}{>{\raggedright\arraybackslash}p{3.6cm}>{\raggedright\arraybackslash}p{9.6cm}>{\raggedright\arraybackslash}p{2.5cm}}
\toprule
Question & Answer & Result\\
\midrule
Can randomizing help? & With $k$ look-alike worlds whose valid-action sets are pairwise disjoint, no randomized policy or team does better than $1/k$ in the worst case. & Theorem~\ref{thm:randombits}\\
Exactly how well, in general? & The best guaranteed success is $1/\tau^*$, the value of a small linear program over the look-alike worlds. Randomizing helps only when valid actions partly overlap. & Theorem~\ref{thm:minimax}\\
How much new state fixes it? & A $b$-bit hint gives at most $1/\lceil k/2^b\rceil$ in the worst case. Anything short of $\lceil\log_2 k\rceil$ bits leaves it at a coin flip or worse. & Corollary~\ref{cor:semanticbits}\\
When do tools help? & A tool result is a new observation. Tools help exactly when some sequence of queries splits the look-alike worlds into groups that each share a valid action. & Theorem~\ref{thm:adaptive}\\
Does ``good enough'' escape it? & No. The same condition holds with ``valid'' replaced by ``within $\varepsilon$ of optimal.'' & Proposition~\ref{prop:epsilon}\\
\bottomrule
\end{tabular}
\normalsize
\end{center}

\begin{theorem}[Randomized aliasing bound]\label{thm:randombits}
\emph{In one sentence: when $k$ look-alike worlds each need a different action, no amount of randomizing does better than one in $k$.}

Precisely, let $x_1,\ldots,x_k$ be worlds that produce the same observation and whose valid-action sets $C(x_i)$ are nonempty and pairwise disjoint. Then any randomized policy that gets no further information has worst-case success probability at most $1/k$. The bound holds equally for one stochastic agent and for a randomized team that receives no new observation of which world it is in.
\end{theorem}

\emph{Proof sketch.} All $k$ worlds induce the same output distribution. Because their valid-action sets are disjoint, the $k$ success probabilities sum to at most one, so at least one is at most $1/k$. Conditioning on the internal random seed gives the same conclusion for randomized multi-agent protocols. Appendix~\ref{app:aliasing} gives the full proof.

If a dashboard makes ten situations look identical and each needs a different intervention, an agent that randomizes, debates, or samples many chains of thought still has worst-case success at most $10\%$. Inference compute and semantic state are different resources, and one cannot pay for the other; extra roles, tokens, or rounds help exactly when they bring in a distinction that splits look-alike worlds (Study~I).

\subsection{What the runtime must keep}\label{sec:exact-state}

The fix follows. Coherence should not depend on every participant rebuilding the whole world in its head. The runtime should hold the distinctions that change which actions are valid, and check proposals against them. How much must it hold?

A runtime that rejects everything is perfectly safe and perfectly useless, so we name three properties of a validator. It is \emph{safe} if it never accepts an invalid action, \emph{makes progress} if it accepts something whenever some action is valid, and is \emph{complete} if it accepts every valid action. A validator that is safe and complete is \emph{exact}: it accepts exactly the valid actions. A state or context is \emph{sufficient} when it makes an exact validator possible. If the runtime remembers only $z=\phi(x)$, let $X_z$ be the worlds that share that summary. The most it can safely accept is the set of actions \emph{guaranteed valid} at $z$, those valid in every such world:
\[
\mathcal B(z)=\bigcap_{x\in X_z} C(x).
\]

\begin{proposition}[What a summary can safely accept]\label{prop:sound}
\emph{In one sentence: a harness that remembers only a summary can safely accept exactly the actions valid in every world with that summary, and can judge every proposal correctly only if the summary never merges two worlds with different valid actions.}

Precisely, suppose the validator reads only $z=\phi(x)$.
\begin{enumerate}[label=(\roman*)]
\item \emph{Safety.} If it is safe, it accepts only actions in $\mathcal B(z)$.
\item \emph{Progress.} It can make progress at $z$ if and only if $\mathcal B(z)\neq\varnothing$.
\item \emph{Exactness.} It can be exact on arbitrary proposals if and only if $\phi(x)=\phi(x')\Longrightarrow C(x)=C(x')$.
\end{enumerate}
\end{proposition}

Take the TeamBench budget of 20 command runs. If the harness remembers only whether any runs are left, it merges ``10 left'' with ``1 left.'' A request for 5 runs is valid in the first world and invalid in the second, so this summary can safely accept only 1-run requests and cannot judge the rest. Remembering the number left makes it exact.

The smallest sufficient state is the quotient
\[
x\sim_C x'\iff C(x)=C(x').
\]
The ``authoritative semantic state'' means: not a copy of the whole world, but every distinction that can change whether a proposal is valid. The model does not need the whole world in its prompt. Instead, the harness can hold that state outside the model and put into each decision's context only what it needs. For context engineering this is the exact target: a context is sufficient precisely when it never merges two worlds that allow different actions.

However, knowing what must be kept does not make it easy to find. Suppose a designer has candidate fields, sensors, or log entries and wants the fewest that make validation exact.

\begin{theorem}[Minimum exact instrumentation is NP-hard]\label{thm:instrument}
\emph{In one sentence: finding the fewest fields that make a context sufficient is NP-hard.}

Precisely, given finitely many states, the valid actions in each, and a list of candidate fields that could be recorded, finding the fewest fields whose values determine the valid actions is NP-hard.
\end{theorem}

Every pair of states with different valid actions must be told apart by some recorded field, and choosing the fewest fields that tell all such pairs apart is Set Cover (Appendix~\ref{app:instrument}). Greedy selection is close to the best polynomial-time approximation, and the problem is a close relative of finding minimal reducts in rough-set theory \cite{skowronrauszer1992}. This is a hardness result about \emph{design}, not about running the system: checking a proposal against a chosen state can be cheap even when finding the smallest sufficient state is hard. It is the first reason that ``just share everything'' is not an answer.

\subsection{Finding and keeping a sufficient context}\label{sec:refine}

Call the set of fields exposed to a validator or an agent its \emph{interface}. The hardness theorem is about the \emph{smallest} sufficient interface; it does not say that a sufficient one is hard to find or to check. A runtime can split the two jobs: a heuristic searches for a compact interface, and an exact check decides whether that interface is sufficient. The check is Proposition~\ref{prop:sound} restricted to a declared scope. Fix a task: a modeled region $X_q$ of possible states, an action interface $A_q$, and the rules in force, which determine $C$. For selected observables $J$ with joint map $F_J$, look for a \emph{witness}: two states
\[
x,x'\in X_q\quad\text{with}\quad F_J(x)=F_J(x')\quad\text{and}\quad C(x)\cap A_q\neq C(x')\cap A_q .
\]
If there is none, $J$ is sufficient for exact validation within that scope. If there is one, it names a distinction the interface has lost, and any candidate observable that separates $x$ from $x'$ repairs it. That gives a loop:
\begin{enumerate}
\item start from the current interface, empty at first;
\item search for a witness, and stop if there is none;
\item add the cheapest candidate that separates the witness, and go back to step 2;
\item then remove fields, most expensive first, keeping each removal only if the check still passes.
\end{enumerate}
This is counterexample-guided abstraction refinement \cite{clarke2000cegar} applied to validity. The abstraction is what the validator or the agent sees. A counterexample is a pair of worlds that the view merges and the rules separate. The output is a \emph{sufficiency certificate}: the interface together with the scope it was checked for.

\begin{proposition}[Certified refinement]\label{prop:refine}
\emph{In one sentence: the loop always stops; it returns a sufficient interface whenever the candidates allow one and says ``unresolved'' otherwise; and what it certifies is safe while its scope, rules, and reads stay current.}

Precisely, let $X_q$ and the candidate fields be finite and fixed, and let the witness search be exact.
\begin{enumerate}[label=(\roman*)]
\item \emph{It stops.} Each round adds a new candidate, so the loop stops after at most as many additions as there are candidates.
\item \emph{It returns the right answer.} If the candidates together are sufficient, the loop returns a sufficient interface from which no single field can be removed. Otherwise it returns an interface that separates everything the candidates can separate and reports \emph{unresolved}; then no choice of candidates makes validation exact.
\item \emph{What it certifies is safe.} If the actual state lies in $X_q$, the certified rules are the rules in force, and the values read are confirmed current at commit (Proposition~\ref{prop:occ}), every action the certified validator accepts is valid.
\end{enumerate}
\end{proposition}

The proof is in Appendix~\ref{app:refine}. In Study~X the loop kept a sufficient context through two or three rule changes per project at 28\% of the cost of sharing every field. The result promises sufficiency, not minimality: an interface no field can leave may still cost more than the minimum, which is NP-hard to find. Greedy Set Cover over the explicit conflict universe is the classical alternative, with its logarithmic guarantee, but it needs every conflicting pair up front. The loop needs only the witnesses it meets. A runtime need not certify a whole interface at every step either. To accept one proposal $a$ at observation $y$, it suffices that $a$ is valid in every modeled state with that observation, which is the per-proposal form of Proposition~\ref{prop:sound}.

\paragraph{Keeping the certificate.} A certificate is only as current as its scope, and updates come in three kinds. A \emph{value} update, such as a spend or a withdrawn approval, moves the state within $X_q$. A certificate over the whole region still holds; only the observations must be re-read, which is the job of the version checks of Section~\ref{sec:transitions}. A \emph{rule} update, such as a new policy, permission, action type, or dependency, changes $C$ or $A_q$ and can break sufficiency itself. It must invalidate the certificate and restart the loop from the old interface. A \emph{history} update changes which futures are valid and is carried by the history state $D$ (Theorem~\ref{thm:selfupdate}). A version check cannot stand in for recertification: it notices that a tracked field changed, never that a new rule made an untracked field necessary. And no procedure certifies against a change outside the model: a variable no one modeled is found only when the model is extended to include it. Study~X runs this loop on a project whose rules change twice (Section~\ref{sec:sufficiency-study}).

\paragraph{Certificates and the structures.} The loop certifies one interface for one scope. In a team the sufficient state is split, translated, and changed, and a certificate must survive each of those maps. This is where the structures of Section~\ref{sec:bridge} enter.
\begin{itemize}
\item \emph{Split.} Each agent's view is certified for the actions that agent may take, because each decision rests on the acting participant's own information (Corollary~\ref{cor:roles}). Certified views must still glue; sufficiency of each view does not imply that the views fit together (Section~\ref{sec:bridge}).
\item \emph{Translated.} A certificate carries across a translation only if the translation is lossless: no translation restores a distinction its source has lost (Proposition~\ref{prop:groupoid}).
\item \emph{Changed.} A certificate over the whole region survives value changes as long as observations are re-read. A runtime that updates its state without re-reading needs an interface that is also self-updating, the history state $D$ of Theorem~\ref{thm:selfupdate}.
\item \emph{Re-scoped.} A change to the rules, the action interface, or the overlaps in $H$ invalidates the certificate.
\end{itemize}
Study~X tests the single-interface case. Per-agent certificates glued across views are the agent experiment of Section~\ref{sec:next-evidence}.
\section{When every agent is right and the team is still wrong}\label{sec:coherence-math}

Section~\ref{sec:model} was about information no one has. This section is about information everyone has, in pieces that each pass every local check and still add up to no possible world. It happens in three ways: pieces compared pairwise, a budget no one owns, and translations between formats. For each, a sharp result says when local checks are enough and what it costs when they are not. Proofs are in Appendices~\ref{app:realize} and~\ref{app:cycles}.

\subsection{Pairwise checks miss loops}\label{sec:realization}

Consider three agents that each own one variable and each reconcile with a neighbor. The A--B check requires $A=B$, the B--C check requires $B=C$, and the A--C check requires $A\neq C$. Each check can be satisfied, and so can any two of them together. No assignment satisfies all three: the first two force $A=C$, and the third forbids it. Replace the Booleans with ``uses schema v2,'' ``the deadline moved,'' or ``these two records are the same customer,'' and this is a familiar failure of multi-agent pipelines: every handoff review passes, and the integrated result is impossible (Figure~\ref{fig:failure-modes}b).

The question underneath is whether local states that agree on every overlap come from one global state; in the notation of Section~\ref{sec:formal}, whether $\rho$ is onto. An assignment of states to scopes for which the answer is always yes, and uniquely so, is called a \emph{sheaf} \cite{robinson2017,ayzenberg2025sheaf}. The sheaf condition gives a runtime an exact target to check; search and repair are separate steps.

The false pass happens when agents compare what each considers \emph{possible} (constraints, plans, ranges) rather than what each has \emph{committed to}; comparing committed values on overlaps catches the triangle, but agents usually exchange the former. (This is the possibilistic contextuality of Abramsky and Brandenburger \cite{abramsky2011sheaf}; Appendix~\ref{app:realize}.)

When are pairwise checks enough? Database theory answers, and the answer depends only on the shape of who overlaps with whom.

\begin{theorem}[Tree-shaped overlaps suffice; Beeri, Fagin, Maier, and Yannakakis \cite{beeri1981properties,beeri1983acyclic}]\label{thm:bfmy}
\emph{In one sentence: pairwise checks are enough exactly when the agents' overlaps form a tree.}

Precisely, if the overlap structure of the local constraints is acyclic (formally, $\alpha$-acyclic: it has a join tree), local constraints that agree pairwise are globally consistent. If it has a cycle, some family of local constraints agrees pairwise and is still not globally consistent.
\end{theorem}

A planner--coder--tester chain is a tree. So is a hub-and-spoke team in which every agent reconciles only with an orchestrator that holds the shared fields. A review loop, or three teams that each share an interface with the other two, is a cycle. Cycles remove the pairwise guarantee, and increasing the checked group size has no universal fixed bound across unrestricted constraint families:

\begin{proposition}[No fixed-size check is enough]\label{prop:nok}
\emph{In one sentence: on cycles, checking groups of any fixed size can still miss a contradiction.}

Precisely, for every $k$ there is a finite set of local constraints in which every $k$ of them can be satisfied together, but all of them cannot.
\end{proposition}

The construction is a parity puzzle on a ring of more than $k$ agents: any proper subset of the constraints can be satisfied, and only all of them together expose the contradiction (Appendix~\ref{app:tseitin}). Some structured problems do have a fixed check size---Helly's theorem gives $d+1$ for convex constraints in $\mathbb R^d$ \cite{matousek2002}---so the useful question is which structure a system has.

\emph{Design rule.} If the agents' overlaps form a join tree, pairwise reconciliation is enough. Cycles remove that guarantee: use a global method appropriate to the constraint class, or restructure the overlaps into a join tree. For translation networks, Section~\ref{sec:representation} gives the exact loop check.

\subsection{A shared budget needs an owner}\label{sec:sharedbudget}

A different failure arises when a constraint lives in no agent's scope. Give $m$ agents a shared budget $B$ of tool calls, test runs, or dollars, each tracking its own spend $s_i$. The agents share no variables, so every local check passes, and nothing stops $\sum_is_i>B$; more checking between agents cannot help, because no view contains the total. This is the TeamBench failure of Study~IV.

\begin{proposition}[A shared budget cannot be protected by independent checks]\label{prop:inconfluent}
\emph{In one sentence: agents that each check only their own spending cannot protect a budget they share.}

Precisely, for $m\ge2$ agents and a budget $B\ge1$, two agents can each make a spend that is valid from the same starting state, while the two spends together exceed $B$. In database terms, the spends are not invariant-confluent \cite{bailis2014}, so protecting the budget requires coordination or a restructured budget.
\end{proposition}

With a budget of 10 and nothing spent, each of two agents may validly spend 6; together they spend 12. Neither agent's check was wrong, because neither agent could see the other's spend.

Two fixes work. An authoritative ledger that checks the total at commit adds a scope that contains the constraint. Escrow quotas $B_i$ with $\sum_iB_i=B$ split the constraint into local ones, at the cost of flexibility: one agent cannot borrow another's unused share. Locks and reservations are further options. Study~IV runs both fixes against an information-only alternative; Appendix~\ref{app:escapes} compares them.

\emph{Design rule.} Name an owner for every constraint that spans agents. Showing the number to everyone is not ownership.

\subsection{Translations that each look right can fail around a loop}\label{sec:representation}

Agents translate constantly: meters to feet, a JSON schema to a SQL table, class $X$ in one ontology to class $Y$ in another, an old function name to a new one. A collection of reversible, meaning-preserving translations is a \emph{groupoid}, or a \emph{group} when the same translations apply to everything. The name says what to keep: not only whether two things are equivalent, but the translation itself. Put agents on the nodes of a graph, and let each edge carry a translation $g_{ij}$ from agent $i$'s format to agent $j$'s, with $g_{ji}=g_{ij}^{-1}$. The system is globally consistent if one underlying object explains every agent's description: there are frames $h_i$ with $g_{ij}=h_ih_j^{-1}$. This is the noiseless group-synchronization problem \cite{lerman2022cycle}.

\begin{theorem}[Loops decide consistency]\label{thm:cycles}
\emph{In one sentence: a network of translations is consistent exactly when going around every loop brings you back where you started.}

Precisely, the translations $g_{ij}$ are globally consistent, meaning there are frames $h_i$ with $g_{ij}=h_ih_j^{-1}$, exactly when their composition around every loop is the identity.
\end{theorem}

In the beam example of Figure~\ref{fig:framework}, Design converts meters to millimeters, Fabrication converts millimeters to inches with the wrong factor $\div25$ instead of $\div25.4$, and Cost converts inches back to meters. Each conversion looks plausible on its own edge, but around the loop $1$\,m comes back as $1.016$\,m, so no single beam length explains all three records.

In the proof, one direction is a telescoping product; the other builds frames along a spanning tree. That construction is also the practical algorithm:

\begin{corollary}[A linear-time consistency check]\label{cor:span}
\emph{In one sentence: checking one edge per independent loop is enough, and it takes linear time.}

Precisely, pick a spanning forest of the agent graph, fix one frame per tree, and propagate frames along the trees. Then check each remaining edge against the frames at its two ends. The translations are consistent exactly when every such check passes. With $|V|$ agents, $|E|$ translation edges, and $c$ connected components, this takes $\beta_1=|E|-|V|+c$ edge checks and $O(|V|+|E|)$ compose-and-compare operations, although the graph can contain exponentially many loops.
\end{corollary}
For the three beam agents, the triangle has three edges and one loop: propagate the length along two edges and compare on the third, one check in total.

Each edge beyond the forest adds one independent constraint; for a finite group, a uniformly random labeling is consistent with probability $|G|^{-\beta_1}$ (Appendix~\ref{app:cycles}). A loop whose translations do not compose to the identity is said to have nontrivial \emph{holonomy}. The triangle of Section~\ref{sec:realization} is the smallest example: write ``same'' as $0$ and ``opposite'' as $1$, and the loop sums to $0+0+1=1\neq0$. In ontology matching the same structure appears as a chain of ``same as'' links that joins two distinct classes of one ontology. No single link looks wrong; only the loop shows the conflict (Studies~VIII and~IX).

\paragraph{Why canonical IDs are not enough.} A tempting shortcut is to map every description to a canonical form---a canonical ID, an embedding, a normalized schema---and compare those. That keeps \emph{whether} two things are equivalent and throws away \emph{how}. It can hide exactly the failure above: every description looks the same after canonicalization while the loop of translations still does not close (Study~VII; Appendix~\ref{app:quotient}). Edits that change the object, such as revising a beam from $1\,\mathrm{m}$ to $1.2\,\mathrm{m}$, are not translations at all. They are state changes, treated in Section~\ref{sec:transitions}.

\emph{Design rule.} Keep the translations, not just the equivalence classes, and check every loop. A spanning forest does this in linear time.

\subsection{Structure, not agent count, sets the price}\label{sec:cost}

How expensive is coherence? The answer depends on how agents interact, not on how many there are.

\paragraph{Exact checking is exponential only in width.} For finite-domain constraints, dynamic programming over a tree decomposition checks global consistency in time roughly $O(nd^{w+1})$ for $n$ variables of domain size $d$ and treewidth $w$ \cite{freuder1990}. For general constraint problems the exponential dependence on width cannot be substantially removed \cite{grohe2007,marx2010} (Appendix~\ref{app:width}). A hundred agents in a pipeline can be cheap to keep coherent; five densely coupled agents can be expensive.

\paragraph{Finding a conflict is cheap; fixing it optimally is not.}

\begin{proposition}[Minimum repair is Max-Cut-hard]\label{prop:z2repair}
\emph{In one sentence: finding a contradiction is cheap, but finding the smallest fix is NP-hard.}

Precisely, for ``same/opposite'' labels on a graph, Corollary~\ref{cor:span} decides consistency in linear time, but finding the fewest labels whose change makes the graph consistent is NP-hard.
\end{proposition}

With every label set to ``opposite,'' the fewest changes equals the number of edges minus the maximum cut, so the problem contains Max-Cut \cite{karp1972,gareyjohnson1979} (Appendix~\ref{app:syndrome}). The same gap holds for general linear constraints, where minimum repair is syndrome decoding \cite{berlekamp1978}.

\emph{Design rule.} Check before commit rather than repair afterward. Detecting an inconsistency at commit is cheap; deciding afterward which agent's output to change can be the hard part.

\section{From theory to a coherence harness}\label{sec:architecture}

This section turns the results into a runtime, the coherence harness: how it sits next to the model, what it checks at commit, what it remembers, why it needs the structures of the introduction, and how the same results organize the team around it.

\subsection{The harness loop}\label{sec:ai-theory}

An LLM policy is useful because it is flexible, and the same flexibility makes its proposals stochastic and hard to certify end to end. The runtime therefore splits the work. Agent $i$ proposes from its context,
\[
\widetilde a_t\sim\pi_\theta(\cdot\mid\Omega_i(x_t),m_t,\mathcal T_i),
\]
where $m_t$ is its memory and $\mathcal T_i$ its tools. The runtime checks the proposal against its authoritative state $z_t=\phi(x_t)$,
\[
V(z_t,\widetilde a_t)\in\{\operatorname{accept},\operatorname{repair},\operatorname{reject}\},
\]
and only accepted changes produce $x_{t+1}$. The checks can cover identity, authority, freshness, budgets, format compatibility, and required artifacts. The model stays stochastic; the rule for what may change does not have to be. Figure~\ref{fig:architecture} shows the split.

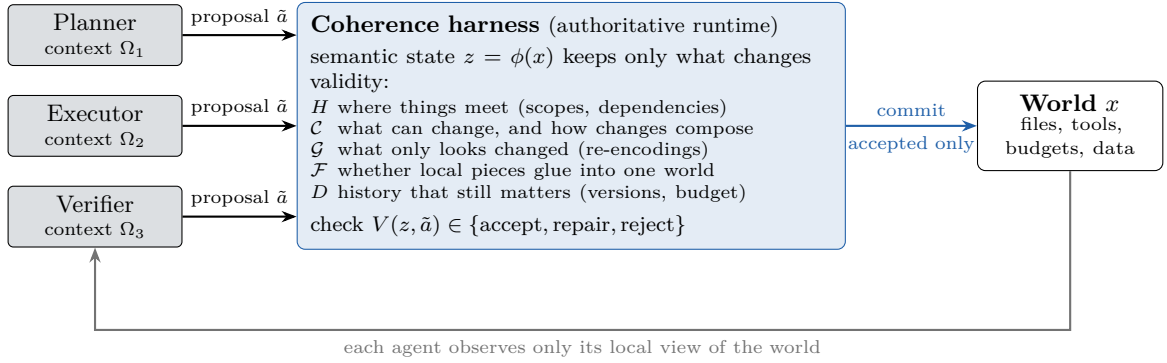
\begin{figure}[t]
\centering
\begin{tikzpicture}[
  font=\small,
  agent/.style={draw, rounded corners=2pt, fill=cohgray!35, minimum width=2.3cm, minimum height=0.75cm, align=center},
  box/.style={draw, rounded corners=3pt, align=left, inner sep=5pt},
  arr/.style={-{Stealth[length=2mm]}, thick},
]
\node[font=\small\bfseries, align=center] at (0,2.2) {Proposal intelligence\\[-2pt]\scriptsize\normalfont stochastic models};
\node[agent] (p) at (0,1.2)  {Planner\\[-2pt]\scriptsize context $\Omega_1$};
\node[agent] (e) at (0,0)    {Executor\\[-2pt]\scriptsize context $\Omega_2$};
\node[agent] (v) at (0,-1.2) {Verifier\\[-2pt]\scriptsize context $\Omega_3$};

\node[box, fill=cohaccent!12, draw=cohaccent, text width=6.9cm] (rt) at (6.3,0) {%
  \textbf{Coherence harness} \footnotesize(authoritative runtime)\\[2pt]
  \footnotesize semantic state $z=\phi(x)$ keeps only what changes validity:\\[1pt]
  \scriptsize
  \begin{tabular}{@{}l@{\ }l@{}}
  $H$ & where things meet (scopes, dependencies)\\
  $\mathcal C$ & what can change, and how changes compose\\
  $\mathcal G$ & what only looks changed (re-encodings)\\
  $\mathcal F$ & whether local pieces glue into one world\\
  $D$ & history that still matters (versions, budget)
  \end{tabular}\\[3pt]
  \footnotesize check $V(z,\tilde a)\in\{\text{accept},\text{repair},\text{reject}\}$};

\node[box, fill=white, text width=2.1cm, align=center] (w) at (12.9,0) {\textbf{World} $x$\\[-1pt]\scriptsize files, tools,\\[-2pt]\scriptsize budgets, data};

\draw[arr] (p.east) -- node[above, font=\scriptsize, sloped] {proposal $\tilde a$} (rt.west |- p.east);
\draw[arr] (e.east) -- node[above, font=\scriptsize] {proposal $\tilde a$} (rt.west |- e.east);
\draw[arr] (v.east) -- node[above, font=\scriptsize, sloped] {proposal $\tilde a$} (rt.west |- v.east);
\draw[arr, cohaccent] (rt.east) -- node[above, font=\scriptsize] {commit} node[below, font=\scriptsize] {accepted only} (w.west);
\draw[arr, black!55] (w.south) -- ++(0,-2.1) -| node[pos=0.25, below, font=\scriptsize] {each agent observes only its local view of the world} (v.south);
\end{tikzpicture}
\caption{The architecture the theory leads to. Stochastic agents propose from their own contexts; the coherence harness keeps only the state that changes validity, checks each proposal, and commits accepted ones to the world. The five rows are the structures $H,\mathcal C,\mathcal G,\mathcal F,D$ of Section~\ref{sec:machinery}; each is needed only when the map it guards (restriction, transition, translation, history) can lose information (Corollary~\ref{cor:layers}).}
\label{fig:architecture}
\end{figure}

The split separates three questions that agent benchmarks usually score at once: \emph{capability} (can the model propose a useful action?), \emph{identifiability} (does the available state determine which actions are valid?), and \emph{commit correctness} (do accepted actions keep the system consistent?). The theory predicts that scaling capability saturates where the bottleneck is one of the other two.

\paragraph{Why bounded models make the separation more valuable.}
Let $\mathcal J_\theta$ be the interfaces that model $\theta$ can actually receive under its context window, bandwidth, privacy, and authority limits, and write
\[
q_\theta(J,x)=\Pr\!\left[\widetilde a\in C(x)\mid F_J(x)\right]
\]
for its single-proposal validity on interface $J$. Two facts hold whatever the model's capability. First, if no interface in $\mathcal J_\theta$ separates two worlds whose valid actions share nothing, no proposal from $\theta$ is guaranteed valid there (Theorem~\ref{thm:aliasing}). Second, commit safety does not depend on $q_\theta$ at all. If the harness's own state $\phi$ is sufficient and it checks proposals against current state atomically, every accepted proposal is valid. A weaker proposer costs rejections and progress, not safety. The harness's state is not bounded by $\theta$'s context, so it can hold what a small model cannot. This is why the separation matters most for bounded models. Compact sufficiency keeps the needed distinctions inside a small context: in Study~X the loop selects 3.5 of about 12 fields. Coherence checks protect those distinctions when no single context can hold them: in Study~IX one full view leaves 0 of 40 conflicts, split views leave 9 of 40, and a network-level check removes all nine. The paper proves these two facts; it does not claim that smaller models always gain more.

\subsection{What the runtime checks at commit}\label{sec:transitions}

\paragraph{Re-expressing state is safe; changing it is not.} A \emph{re-expression} (converting units, translating a schema) leaves the underlying object unchanged; a \emph{state change} (spending budget, editing a design, overwriting a file) does not. If agents only re-express their state and translations agree on overlaps, the result stays globally consistent. If they change state, each change can be locally valid and every overlap can agree while the combination is no valid world, as with the shared budget of Section~\ref{sec:sharedbudget}. So state changes need a commit-time check and re-expressions do not (Proposition~\ref{prop:preserve}, Appendix~\ref{app:preserve}); by induction, a trajectory whose every committed step is a re-expression or was checked at commit stays coherent. This is the line between the groupoid $\mathcal G$ and the category $\mathcal C$.

\paragraph{Keep only the history that changes the future.} Two histories can share one memory state exactly when every continuation is valid after one if and only if it is valid after the other; by the Myhill--Nerode theorem \cite{hopcroft2006}, the smallest exact monitor has one state per such class. For a budget $B$ and $k$ obligations that must be done before \texttt{finish}, the smallest state is the obvious pair, remaining budget and completed obligations: $(B+1)2^k+2$ states, but only $\lceil\log_2(B+1)\rceil+k$ bits plus a terminal flag (Appendix~\ref{app:nerode}). Provenance, such as which version an artifact was derived from, belongs in memory exactly when it changes which futures are valid. A stale derivation is the standard case (Study~III).

\paragraph{Stale reads and atomic commit.} An agent reads state, thinks for several turns, and proposes a change; by then the state may have moved. The classical fix is optimistic concurrency control \cite{kung1981}: record the version of everything the proposal depended on, and at commit atomically check that none of it has changed before applying the change.

\begin{proposition}[Version checks give one-at-a-time semantics]\label{prop:occ}
\emph{In one sentence: if every proposal records what it read and each commit checks and writes in one atomic step, the result is as if the agents had taken turns.}

Precisely, suppose every proposal records every state item it depended on, every write advances the version of the item it writes, and the version check and the writes happen atomically. Then the committed changes are equivalent to running them one at a time in commit order.
\end{proposition}

For an LLM agent the first condition is the hard one. The read set is every authoritative item that entered the model's context and influenced the proposal, not only the items the final tool call touches. A quote the agent read five turns ago belongs to the read set of the order it places now; this is what ordinary $\tau^2$-bench checks miss (Study~V). The model's sampled output is part of the transaction, so the equivalence is to serial execution with the same samples (Appendix~\ref{app:occ}).

\subsection{From sufficiency to coherence: why sheaves, groupoids, and categories}\label{sec:bridge}

The impossibility theorem says only that the runtime must keep every distinction that changes which actions are valid (Proposition~\ref{prop:sound}). In one equation, a state $\phi$ is sufficient when validity can be computed from it:
\[
C(x)=\bar C\bigl(\phi(x)\bigr)\quad\text{for some map }\bar C .
\]
In words: everything needed to judge an action can be read off $\phi$. In a simple system one ledger field does that. The structures appear because in a multi-agent system that state is never in one place and never still. It is split into the agents' views, translated between their formats, and changed by actions, and each of these steps can lose a needed distinction. Each structure is the condition under which one step loses nothing: the sheaf condition for the restriction to views, the groupoid for translation, and the category with its history state for action. Appendix~\ref{app:bridge} states each as a formal result (Propositions~\ref{prop:localsuff} and~\ref{prop:groupoid}, Theorem~\ref{thm:selfupdate}). The structures are therefore not a second theory placed beside the impossibility theorem. Each is that theorem applied at one map.

\begin{center}
\small
\begin{tabular}{>{\raggedright\arraybackslash}p{2.7cm}>{\raggedright\arraybackslash}p{4.6cm}>{\raggedright\arraybackslash}p{4.3cm}>{\raggedright\arraybackslash}p{1.6cm}>{\raggedright\arraybackslash}p{1.6cm}}
\toprule
Map & What goes wrong in an agent system & What the runtime must guarantee & Structure & Studies\\
\midrule
world $\to$ agents' views & two worlds look the same to every agent; or views fit pairwise but match no world & views jointly separate worlds that allow different actions; committed views glue & sheaf ($H$, $\mathcal F$) & I, IV, IX\\
format $\to$ format & a lossy translation merges objects; translations disagree around a loop & translations are reversible and close around every loop & groupoid ($\mathcal G$) & VII--IX\\
state $\to$ next state & the runtime must re-read the world after each step to know what is valid & its state updates itself, and updates compose & category ($\mathcal C$) & IV, V\\
history $\to$ memory & memory forgets what still matters, or keeps what never will & keep history exactly when it changes valid futures & history ($D$) & III, V\\
\bottomrule
\end{tabular}
\normalsize
\end{center}

\paragraph{Splitting state into views needs gluing.} A runtime that reads only the agents' local views can be exact only if the views jointly tell apart every two worlds that allow different actions; this is the impossibility theorem applied to the map from the world to the views. And it commits only states that some real world produces exactly when it rejects every combination of views that matches pairwise but comes from no world. These are the two conditions that define a sheaf, with one weakening that matters in practice: the views need to separate worlds only by which actions they allow, not in every detail (Proposition~\ref{prop:localsuff}). Section~\ref{sec:coherence-math} showed both failures.

\paragraph{Translating between formats needs a groupoid.} A translation cannot restore a distinction its source format already lost: if agent $i$'s format merges two objects, every format translated from it merges them too. When every format is lossless, the translations between them are forced: each can be undone, and composing them around any loop returns the identity. Conversely, any family of reversible translations that closes around every loop comes from lossless encodings of one object (Proposition~\ref{prop:groupoid}). That is exactly a groupoid with consistent loops, and Corollary~\ref{cor:span} checks it.

\paragraph{Acting needs a self-updating state.} A runtime that must re-read the world after every action to know what is now valid is not keeping state; it is polling. For the runtime to update its own state, the state must keep apart any two worlds that some future sequence of actions would tell apart. A budget shows the difference. With one or two units left, the same actions are valid now, but after one spend they are not. ``May spend now'' is exact for one step; the remaining count is the smallest state that updates itself (Theorem~\ref{thm:selfupdate}, the Myhill--Nerode construction applied to valid actions \cite{hopcroft2006,milner1989}). Because the updates compose---updating for $a$ and then for $b$ is updating for $ab$---actions act on the runtime's state as a category. The history state $D$ is the same construction indexed by histories.

\begin{corollary}[Where each structure is needed]\label{cor:layers}
\emph{In one sentence: each structure is needed exactly when its step is present: a gluing check for split views, consistent translations for several formats, and a self-updating state for actions.}

Precisely, consider a runtime that must accept exactly the valid actions and commit only states some world produces. If it accepts arbitrary combinations of local views at commit, it needs the gluing check. If agents use several lossless formats, it needs reversible, loop-consistent translations. If it must update its state under actions without re-reading the world, it needs the self-updating state. Conversely, with one scope there is nothing to glue, with one format there is nothing to translate, and when no action changes the state the self-updating state reduces to the sufficient state of Proposition~\ref{prop:sound}.
\end{corollary}

Remove a map and its property drops out; add the map and exact validation requires the property (proofs in Appendix~\ref{app:bridge}). The corollary constrains properties, not machinery: a version counter, dependency engine, or residual solver meets the condition in its own domain and should tie the general runtime there, as in Studies~III and~VI. The structures are the general way to meet the conditions across boundaries where no such mechanism exists. The same pattern, a requirement of sufficient information followed by structures that preserve it, appears in team decision theory, atlases, filtering, and version control (Appendix~\ref{app:bridge}). The theorem does not force this design alone: weakening the goal, querying, restructuring state, or accepting probabilistic success are alternatives with their own costs (Appendix~\ref{app:escapes}).

\paragraph{Sufficiency and coherence are two obligations.} The restriction map $\rho$ from worlds to local views can fail in two ways (Section~\ref{sec:formal}): it can merge worlds that allow different actions (aliasing), or admit views that match on every overlap and come from no world (a gluing failure). The first is a question of \emph{sufficiency}, whether the represented state keeps the distinctions decisions need; the second of \emph{coherence}, whether the represented pieces fit together. Neither implies the other (Proposition~\ref{prop:localsuff}). If every agent records a budget of \$10{,}000 and an instruction no one modeled has cut it to \$5{,}000, every overlap agrees and every gluing check passes, because gluing checks only what was represented. Conversely, in the shared budget of Section~\ref{sec:sharedbudget} every role can see the live count and the combined spends still break it. A harness needs both checks: Proposition~\ref{prop:sound} and the refinement loop of Section~\ref{sec:refine} for sufficiency, the structures of this section for coherence.

\subsection{The framework assembled: \texorpdfstring{$\mathfrak X=(H,\mathcal C,\mathcal G,\mathcal F;D)$}{X = (H, C, G, F; D)}}\label{sec:machinery}

The five components of the introduction are now runtime duties: $H$ records scope, authority, and dependencies; $\mathcal C$ gives each action its preconditions and effects; $\mathcal G$ keeps identity and translations; $\mathcal F$ checks overlaps and realization at commit; and $D$ keeps versions, obligations, and resources. For readers who know category theory: $\mathcal G$ makes the assignment of states to the scopes of $H$ a groupoid-valued presheaf, $\mathcal F$ is its descent condition, and $\mathcal C$ is a category-valued prestack with that groupoid core (Appendix~\ref{app:bridge}); the experiments use only the concrete checks. None of the five reduces to the others. Appendix~\ref{app:separation} gives a countermodel for each: identical topologies with different transitions, identical local operations with different realizability, identical fields that name different objects, and identical values of which only one is still valid. The runtime need not reproduce the model's reasoning or replace domain solvers. It must own the distinctions required by Theorem~\ref{thm:aliasing} and Proposition~\ref{prop:sound}, check global consistency where local checks are not enough, and check state changes at commit.

\subsection{Designing the team: owners, dependencies, and parallel work}\label{sec:team-design}

Multi-agent systems are usually organized by job title, by intuition: a lead oversees, engineers build, pricing waits for the design. The framework turns each such choice into a question about state. For role $i$, let $W_i$ be the authoritative items its proposals write, and $R_i$ every item whose value can change whether its proposals are valid, including the scope of every constraint they must respect. In the beam example, the Cost agent reads the beam length the Design agent writes, and both read the shared budget.
\begin{enumerate}
\item \textbf{Give each role the distinctions its decisions need.} The harness can reject an invalid proposal but cannot make a role propose a valid one. Role $i$ can always propose a valid action exactly when every fiber of its own context has a common valid action (Theorem~\ref{thm:aliasing}), and its information is its own, not the union of all views (Corollary~\ref{cor:roles}). Its sufficient context can be found by refinement (Section~\ref{sec:refine}).
\item \textbf{Give every cross-role constraint an owner.} A constraint whose scope lies in no single role's $W_i$, and whose actions are not invariant-confluent, cannot be protected by independent checks (Proposition~\ref{prop:inconfluent}; \cite{bailis2014}). It needs a scope that contains it and is checked at commit (a ledger) or a split into local constraints (escrow). In this sense \emph{leading} a constraint means holding its whole scope and refusing violating commits, which is the harness's job; a lead agent that receives only summaries is one more $\Omega_i$ (Study~IV).
\item \textbf{Decide parallel or sequential from read and write sets.} If neither role writes what the other reads or writes ($W_i\cap R_j=W_j\cap R_i=W_i\cap W_j=\varnothing$), neither can invalidate the other, and parallel work is free; these are Bernstein's conditions \cite{bernstein1966}, with $R_i$ widened to everything that bears on validity. When Design writes what Cost reads, three designs are correct. Cost can wait for Design, which costs latency. The two can run in parallel with Cost's read versions rechecked at commit (Proposition~\ref{prop:occ}), which costs rework when they conflict. Or the dependency can be restructured: Cost quotes a named design version whose provenance $D$ records (Study~III), or escrow shrinks each role's budget read to its own share. ``Pricing waits for engineering'' is one point on that trade-off.
\item \textbf{Shape who must reconcile with whom.} If the overlap structure $H$ has a join tree, checking each handoff is enough (Theorem~\ref{thm:bfmy}); with cycles, certify the cycle space at $\beta_1$ checks (Corollary~\ref{cor:span}) or let a hub own the shared fields, which turns $H$ into a star. This is why hub-and-spoke orchestration often works, and it says what the hub must hold. Since exact checking grows exponentially with the width of $H$ (Section~\ref{sec:cost}), the target is low width, not few agents.
\item \textbf{Make translations between role vocabularies explicit.} Where two roles encode one object differently, such as engineering units and pricing SKUs, the handoff carries an explicit reversible translation whose loops are certified (Studies~VII--IX).
\end{enumerate}
Each step matters only where its condition is present, and none says which role does the work best, which is a matter of capability. What they replace is guesswork: ``should the tech lead own this?'' becomes ``which constraint spans these roles, and who holds its scope?'' Individual steps have direct evidence in Studies~III--V and VII--IX; testing the procedure as a whole is an open experiment (Section~\ref{sec:next-evidence}).

\section{Empirical program}\label{sec:empirical}
\subsection{Overview: ten studies, three kinds of evidence}\label{sec:theory-exp}

Each study targets one claim of the theory. The agent studies and the real-data study could have come out the other way; the deterministic studies confirm that a mechanism behaves as proved.

\paragraph{One model, held fixed.} All agent studies use one LLM, \texttt{gpt-6-sol}, with the reasoning effort stated in each study. The impossibility theorem binds every policy however capable; each causal study holds the model fixed and changes only what the agents see or what the runtime enforces; and Study~I shows the model solves every task whose deciding event it can see. So every comparison measures the state, not the model. The hidden-fact result also holds on two models from a second provider (Section~\ref{sec:aliasprobe}). Every agent panel was frozen before its first scored call and ran under a spending cap fixed in its preregistration.

\begin{center}
\small
\begin{tabular}{p{2.9cm}p{3.6cm}p{3.7cm}p{4.3cm}}
\toprule
Study & Claim tested & Evidence & Result\\
\midrule
I.\ Aliasing probe & more reasoning cannot recover a hidden distinction (Thms.~\ref{thm:aliasing}, \ref{thm:randombits}, \ref{thm:adaptive}) & agent, 40 paired tasks, frozen design & observed 12--17/40 across aliased arms; none significantly above $1/3$; one authoritative sentence gives 40/40\\
II.\ TeamBench roles & role decomposition can help under information asymmetry & reanalysis of released runs, 164 matched cells & removing Planner or Verifier costs 11--14 pass-rate points; team vs.\ fully informed solo not significant; descriptive, not causal\\
III.\ Stale derivations & current values alias fresh and stale state; a dependency engine that keeps provenance equals the runtime (Thm.~\ref{thm:selfupdate}, Prediction~4) & deterministic synthetic pipeline, 8 seeded cases & unsafe commits: local checks 8/8; dependency closure and runtime both 0/8\\
IV.\ TeamBench budget & a shared invariant needs ownership; escrow can restructure it (Prop.~\ref{prop:inconfluent}) & agent, three panels: 27 + 60 + 15 runs & violations 5/5 ordinary, 4/5 visible/query, 0/5 ledger or escrow\\
V.\ $\tau^2$-bench Telecom & distributed state is common; ordinary checks tie where they own the invariant, and stale reads and silent reverts need the runtime & structural audit, 114 tasks; agent, preregistered, 222 episodes & 97/114 initialize both actors' state; predicted-tie tasks: both guards 0 violations (unguarded 14/34); other tasks: conventional 24/25, coherence 0/25; silent-revert reward 0.07 vs 1.00\\
VI.\ Sobieski & classical solvers equal the runtime on their home ground (Prediction~4) & deterministic, 86 + 40 cases & ties with Gauss--Seidel and version-aware restart\\
VII.\ Cycle suite & a spanning forest certifies representation coherence (Cor.~\ref{cor:span}) & deterministic, 1{,}800 instances & forest 100\% at $\beta_1$ checks; local, pairwise, canonical-ID 0\%; short loops $\le$66\%; random loops 94.5\%\\
VIII.\ OAEI networks & pairwise-invisible identity collisions occur under a unique-name diagnostic & public matcher outputs, preregistered replication & 23 of 45 networks, 2018--2024; 5 of 7 systems in 2025\\
IX.\ Split views & split views need a network-level check & agent, frozen, 60 tasks & split views 9/40 hidden collisions; network check 0/40, but repair lowers mean F1 to 0.774\\
X.\ Sufficiency upkeep & refinement finds and keeps a sufficient interface as rules change (Props.~\ref{prop:sound}, \ref{prop:refine}) & deterministic, frozen, 200 instances with two or three rule changes & loop: 0 invalid accepts, 0 avoidable blocks, 28\% of all-field cost, 1.11$\times$ the exact minimum; fixed interface: 3{,}961 invalid accepts or 2{,}618 valid actions blocked\\
\bottomrule
\end{tabular}
\normalsize
\end{center}

\paragraph{How to read the evidence.} The runtime wins where the deciding state is missing or unowned (Studies I, IV, V, IX) and ties where a conventional mechanism already owns it (Studies III, VI, and half of V). Both halves matter: a theory that only predicted wins could not be wrong. Studies I, IV, IX, and the Study~V revision panel are the causal agent experiments; Study VIII measures prevalence in real data; Studies III, VI, VII, and X are deterministic checks; Study II and the Study~V audit are descriptive.

\subsection{Study I: more reasoning versus one authoritative distinction}\label{sec:aliasprobe}

Theorem~\ref{thm:randombits} predicts that on an aliased fiber with disjoint valid sets, no policy that reasons, samples, or deliberates on the same observation succeeds with probability above $1/k$; one authoritative distinction resolves the task; and a tool helps exactly when it splits the fiber (Theorem~\ref{thm:adaptive}). We built a benchmark in which the aliasing is exact. Each task is a 60-event revision log in which three teams commit, revert, or re-apply changes to one shared parameter; the agent must name the current head, and the commit server rejects any other parent. The final event is one of three operations that leave three different heads. In the \emph{identifiable} version it is visible; in the \emph{aliased} version it is redacted, so the prompt is identical across all three truths and chance is $1/3$. Forty paired tasks were frozen before any scored call.

\begin{center}
\small
\begin{tabular}{lrr}
\toprule
Arm & Identifiable & Aliased\\
\midrule
Low / medium / high reasoning effort & 40/40 each & 12 / 17 / 14 of 40\\
Planner + verifier & 40/40 & 13/40\\
Five-sample majority vote & 40/40 & 13/39\\
Query tool that cannot reveal the event & 40/40 & 13/40\\
Query tool that can read any event & 40/40 & \textbf{40/40}\\
One authoritative fact (``the current head is $X$'') & 40/40 & \textbf{40/40}\\
\bottomrule
\end{tabular}
\normalsize
\end{center}

Every arm solves every identifiable task. On aliased tasks no arm that keeps the same observation significantly exceeds $1/3$ (best one-sided binomial $p=0.14$), whether it reasons longer, deliberates in two roles, or votes over five samples; this does not establish equivalence to chance, which comes from the construction and the theorem. Stated confidence drops from 1.0 to 0.35--0.45. One sentence naming the head gives 40/40 (23 tasks gained over medium effort, none lost; McNemar $p=1.2\times10^{-7}$), as does a tool that can read the redacted event, while a matched tool returning only visible notes stays at 13/40 ($p=7.5\times10^{-9}$). All four preregistered hypotheses hold; one majority-vote cell failed on a transient API error, which cannot change the outcome. This is the paper's central claim in its cleanest form: more reasoning, more roles, and more samples on the same observation leave accuracy at the information-theoretic level, while one authoritative distinction, supplied directly or through a query that splits the fiber, solves every task.

\paragraph{Replication on a second provider.} A companion study \cite{heng2026symmetry} drew 80 fresh tasks from the same generator and ran them on Claude Opus~5.5 and Claude Haiku~5.5. Stating the head solved 80 of 80 on both models. Without it, no arm reached the chance level of 26.7 of 80: Opus scored 13--21 and Haiku 0--6. Haiku did not guess at random. In 76 of 80 tasks it named the head as it stood before the hidden event, assuming that event changed nothing. A weaker model fills the missing distinction with an assumption, and here the assumption is always wrong.

\subsection{Study II: TeamBench role structure}\label{sec:teambench-roles}

TeamBench separates Planner, Executor, and Verifier permissions and grades artifacts deterministically \cite{teambench2026}. We reanalyzed its released runs (six models, 28 tasks; 164 model--task cells matched across five conditions); no new runs are involved. Removing the Planner lowers pass rate by 11--14 points and removing the Verifier by 11--12 (sign tests $p=0.003$ and $p=0.0002$), and the full team gains 13.4 points over a single agent that lacks the specification. The default grader's 35.4\% versus 15.9\% gap between the team and a fully informed single agent is mostly an attestation effect: single agents often never write the Verifier's sign-off. Counting attestation-only failures as passes, as TeamBench's leaderboard does \cite{teambench2026}, leaves 36.0\% versus 31.1\% (not significant, $p=0.18$), in line with TeamBench's own finding that its full team is not significantly better than a solo agent on average \cite{teambench2026}. Roles carry weight when the specification is split, but the panel varies structure together with inference budget, so it is descriptive; Study IV isolates the state hypothesis.

\subsection{Study III: stale derivations, where history is the missing state}\label{sec:stale}

Two states with identical current values can differ in whether they are still valid, so a check that sees only current values is aliased; the distinction lives in provenance, part of $D$ (Theorem~\ref{thm:selfupdate}). A conventional mechanism that tracks provenance should match the runtime (Prediction~4). In a deterministic pipeline we built, eight sources each feed three derived artifacts that record the source revision they came from. Between planning and commit, one source is reissued as a new revision with the same value, so three artifacts cite a superseded revision while every value and every local check stays correct. Local checks and a changed-source closure computed at planning time commit unsafely in 8 of 8 cases; a dependency-closure guard and the coherence harness, which both reread provenance at commit, commit unsafely in 0 of 8. The first failure is aliasing; the second is the stale-read condition of Proposition~\ref{prop:occ}; the tie between the last two is the prediction. The instrument confirms the mechanism, not a natural rate; Study~V measures the agent-level counterpart.

\subsection{Study IV: who owns a shared budget on TeamBench}\label{sec:fixed-budget}

Proposition~\ref{prop:inconfluent} predicts that a shared budget cannot be protected by the agents themselves, however well informed, and that two remedies work: an owner that refuses the violating commit, or per-role quotas. Three panels, each frozen before scoring, hold \texttt{gpt-6-sol} at medium effort fixed on three hard TeamBench tasks: a budgeted workflow (LH2), specification-to-tests (TEST1), and leader election (DIST5, a control).

\paragraph{Panel 1: a guard on the wrong variable.} In 27 runs a ledger guard made no difference (strict passes 1/9 for both ordinary and guarded teams), because guard, task, and grader counted different things: the generator drew limits from 15 to 30 while the grader fixed 20, and the guard counted an internal ledger the task already self-limited. A post-scoring audit found the real failure: teams issued 22--67 shell commands against a limit of 20. A guard constrains the state it owns and nothing else (Proposition~\ref{prop:sound}).

\paragraph{Panel 2: the aligned intervention.} The budget was pinned at 20 calls to TeamBench's \texttt{run} tool summed across the team; the guard refused a call once the total reached 20; the evaluator counted the same unit; and the attestation gate required a \emph{pass} verdict plus a successful rerun of the visible check. A \texttt{visible\_state} arm exposed the same live state without enforcing it. Model, prompts, tools, caps, tasks, and graders were fixed across team arms; the solo arm is a non-causal reference. All 60 runs completed.

\begin{center}
\small
\begin{tabular}{lrrrr}
\toprule
Arm (15 runs) & Strict & Mean partial & LH2 over 20 & TEST1 bad attestation\\
\midrule
Ordinary team & 0/15 & 0.820 & 5/5 & 5/5\\
Budget-matched solo & 0/15 & 0.873 & 0/5 & 5/5\\
\texttt{visible\_state} & 0/15 & 0.860 & 4/5 & 3/5\\
\texttt{coherent\_full} & 2/15 & 0.865 & \textbf{0/5} & \textbf{2/5}\\
\bottomrule
\end{tabular}
\normalsize
\end{center}

Every enforced LH2 run stopped at exactly 20 calls, refusing 1--32 surplus attempts. Ordinary teams issued 57--72; showing the live count lowered the median from 70 to 29 but still left 4/5 violations. Enforcement removes violations against both the ordinary team (5/5 to 0/5, one-sided Fisher $p=0.004$) and the visible count (4/5 to 0/5, $p=0.024$). Progress did not collapse: on LH2 \texttt{coherent\_full} scored 0.874 mean partial versus 0.860, on TEST1 0.820 versus 0.740, and on the DIST5 control 0.900 versus 0.860.

\paragraph{Panel 3: other ways to protect the budget.} On LH2 with the same model, seeds, and grader, a live counter without refusal still overspent in 4/5 runs; conservative stopping (refuse once remaining budget would fall below 5) prevented overspend at the cost of unused slack; and fixed per-role escrow (executor 14, verifier 6) reached 0/5 violations with mean partial 0.960 and 4/5 strict passes, against 0.874 and 2/5 for the ledger. That progress gap is descriptive, since the reused ledger arm kept an attestation gate the new arms turned off. For this additive budget, splitting it into quotas is enough.

\paragraph{What the panels show.} Information is not ownership: agents given the exact live count, in a team that already had a Planner and a Verifier, still overspent (Proposition~\ref{prop:inconfluent}, Corollary~\ref{cor:roles}). Ownership works in either form, refusal at commit or escrow, without lowering observed progress in these runs. And a guard is only as good as the state it owns: before asking whether a guard works, check that it guards the variable the task is scored on (Panel~1).

\subsection{Study V: \texorpdfstring{$\tau^2$}{tau2}-bench Telecom, structural audit and revision panel}\label{sec:tau2}

In $\tau^2$-bench Telecom the user controls the device, the assistant controls the carrier, and the environment grades the joint result \cite{tau2bench2025}. In the pinned v0.2.0 corpus, 97 of 114 tasks initialize both sides' state, 71 declare golden actions for both actors, and 59 assert postconditions in both environments: dual-control state is central to the benchmark.

\paragraph{Revision panel.} We injected three kinds of failure into public Telecom tasks: a \emph{stale read}, where a carrier fact the assistant read changes before it writes; an \emph{authority change}, where a valid action becomes invalid mid-conversation; and a \emph{silent revert}, where a committed carrier change is undone. There are 12 injected tasks on 7 base tasks, plus the 7 bases as controls. The agent, user simulator, loop, and evaluator are unmodified (\texttt{gpt-6-sol}, medium effort for the agent, low for the user). Three arms differ only in how write calls are handled: \emph{unguarded}; \emph{conventional}, which checks three policy preconditions on the current carrier state; and \emph{coherence}, which adds read-set validation (Proposition~\ref{prop:occ}) and a monitor that sends one notice if a committed postcondition stops holding. The frozen prediction is a split: where the change breaks a current precondition, conventional and coherence tie; where a read went stale or a committed change stopped holding, only the coherence arm can see it. All 222 episodes completed.

\begin{center}
\small
\begin{tabular}{lcccccc}
\toprule
& \multicolumn{3}{c}{Target violations / exposed} & \multicolumn{3}{c}{$\tau^2$ reward}\\
\cmidrule(lr){2-4}\cmidrule(lr){5-7}
Tasks & Unguarded & Conventional & Coherence & Unguarded & Conventional & Coherence\\
\midrule
Predicted tie (7 tasks) & 14/34 & 0/35 & 0/34 & 0.63 & 0.86 & 0.86\\
Predicted coherence (5 tasks) & 25/25 & 24/25 & \textbf{0/25} & 0.28 & 0.44 & \textbf{1.00}\\
\quad stale read (S3, S4) & 10/10 & 10/10 & 0/10 & 0.70 & 1.00 & 1.00\\
\quad silent revert (P2--P4) & 15/15 & 14/15 & 0/15 & 0.00 & 0.07 & 1.00\\
Uninjected controls (7 bases) & --- & --- & --- & 0.93 & 1.00 & 1.00\\
\bottomrule
\end{tabular}
\normalsize
\end{center}

The split is confirmed. Where an ordinary check owns the invariant, both guarded arms stop every violating write with the same reward. Under a stale read the conventional arm committed on the old figure in 10 of 10 episodes, every write valid at call time, and the agent never re-read on its own; the $\tau^2$ grader does not penalize these commits, so the benchmark does not see the violation. After a silent revert the unguarded and conventional arms scored 0.00 and 0.07, believing the job done, while the coherence arm scored 1.00 in 15 of 15: it sent one notice and the agent redid the change, which nothing in the design forced. What decides this result is ownership of the committed postcondition, not this implementation: a reconciliation loop that owns it would tie the harness, as Prediction~4 says. No correct action was rejected in any control. Three predicted-tie tasks fell short of their design and are reported as run: in two (S1, A3) the agent re-read the changed fact itself, so no arm violated, and in one (P1) the frozen success criterion was too narrow, so every arm scored 0. The panel injects the failures, so it measures what the mechanism does when they occur, not how often they occur.

\subsection{Study VI: Sobieski, where a classical solver already owns the state}\label{sec:sobieski}

Prediction~4 says that where a conventional mechanism already owns the relevant state, the runtime should tie it: no better, no worse. The Sobieski supersonic business jet is a standard multidisciplinary design benchmark with exact public equations \cite{gemseo2026problem,gemseo2026tutorial,sobieski1998}. Four deterministic disciplines (structure, aerodynamics, propulsion, mission) play the agents, and three of them feed each other in a cycle. The disciplines describe one aircraft only when every shared value equals what its consumer assumed, which is the gluing condition $\mathcal F$.

On 86 cases, stopping after a fixed one or three passes over the disciplines never reached agreement: every run committed an aircraft whose structure was sized for a lift its aerodynamics no longer produced, with range off by about 0.5\% after one pass. Classical Gauss--Seidel iteration to convergence and the coherence guard both reached agreement in all 86 cases with identical work, about nine times the cost of a single pass; that is the price of being right, because a converged residual is $\mathcal F$ in numerical form. On 40 cases in which the design was revised partway through a pass, keeping the mixed result or iterating two more passes failed every case, while the runtime's provenance guard and a conventional version-aware restart both succeeded in all 40 with identical work, because the version number is $D$ in its simplest form. The failure the theory describes occurs in classical engineering, and the general runtime reduces exactly to the classical tools where they own the state (Appendix~\ref{app:sobieski}).

\subsection{Study VII: checking that translations agree around every loop}\label{sec:cycle-suite}

Section~\ref{sec:representation} claims that a translation error can be invisible to every check of one agent or one edge, including canonical IDs; that short loops are not enough; and that all loops can be checked with $\beta_1$ comparisons along a spanning forest. The instances are sign-synchronization problems on twenty graphs (rings, complete graphs, grids, the Petersen graph, random graphs; $\beta_1$ from 1 to 25), 1{,}800 in all and frozen before scoring. The primary corruption flips one edge label on a loop, leaving every node and edge locally valid; control corruptions are visible to simpler validators.

\begin{center}
\small
\begin{tabular}{lrr}
\toprule
Validator & Flipped-edge errors caught & Mean loop checks (coherent)\\
\midrule
Local / pairwise / canonical-ID (quotient-first) & 0\% & 0\\
Bounded cycles ($k=3/4/6$) & 18.3\% / 43.8\% / 66.2\% & 1.1 / 4.0 / 14.6\\
Sampled cycles (budget $\beta_1$) & 94.5\% & 4.2\\
Spanning forest / groupoid descent & \textbf{100\%} & 4.2 / 18.5\\
\bottomrule
\end{tabular}
\normalsize
\end{center}

No validator raised a false alarm. Checks that never look at a loop catch nothing, including canonical IDs, the common engineering shortcut, though each catches the control it can see. A bounded-cycle check catches an error exactly when it lies on a loop of length at most $k$: raising $k$ from 3 to 6 multiplies the checks thirteenfold and still misses a third, the executable form of Proposition~\ref{prop:nok}. Random loops at the forest's budget catch 94.5\% overall but only 76.7\% on the hardest graph, with no indication of what was missed. The spanning forest catches all 600 errors with exactly $\beta_1$ checks, and an independent union-find implementation agrees on every instance. It is the only validator here that is both complete and cheap: one comparison per independent loop, a number set by the graph's shape rather than by its exponentially many loops. The errors are planted, so the zeros confirm the theorem rather than estimate a rate; Study~VIII supplies the rate.

\subsection{Study VIII: holonomy in OAEI Conference alignment networks}\label{sec:oaei}

Study~VII planted the error; Study~VIII asks how often it occurs in real systems. A chain of ``same as'' links that does not close around a loop (nontrivial holonomy) is known in data integration: \texttt{owl:sameAs} chains can join two identifiers of one dataset \cite{demelo2013}, and alignments of three ontologies can be pairwise consistent but globally inconsistent \cite{kutz2010whispers}. We measure it systematically under a preregistered design. On the OAEI Conference track \cite{oaei2025conference}, matching systems produce pairwise equivalence alignments among seven OntoFarm ontologies \cite{zamazal2017ontofarm}. We flag a chain of ``same as'' links when it makes two distinct entities of the \emph{same} ontology equivalent, although every individual alignment looks one-to-one (an explicit \emph{unique-name diagnostic}). This is a structural warning, not an OWL contradiction, and its relation to the curated answers is measured, not assumed.

In 2025, five of seven systems' networks contain collisions no pairwise check can see (50 pairs with witness chains of length at least 3), and correspondences on those chains are absent from the curated reference \texttt{ra1} at 41.7\% versus 27.0\% elsewhere (one-sided Fisher $p=2.1\times10^{-4}$, optimistic because chains overlap). A triangle check catches 78\% of them. Typical cases merge a class with its subclass, such as $\texttt{sigkdd\#Author}\equiv\texttt{sigkdd\#Author\_of\_paper}$ in five systems. A second frozen design replicated this on every distinct complete network from 2018 to 2024: 23 of 45 contain a pairwise-invisible collision (95\% CI 37--65\%), in every one of the seven years; chain correspondences are more often absent from \texttt{ra1} in 21 networks versus 2 (sign test $p=3.3\times10^{-5}$); and collision density correlates with disagreement with the reference (Spearman $-0.63$), a candidate reference-free quality signal. Absence from the reference only bounds error, because the reference is incomplete.

\subsection{Study IX: agents curating mappings, with full and split views}\label{sec:split-view}

We built 40 conflict tasks from real OAEI components containing a pairwise-invisible conflict, plus 20 conflict-free controls, and froze two experiments before any scored call. In the first, one agent sees every candidate correspondence and is told the transitivity and unique-name rules. It resolved every conflict itself (0/40 left), so no runtime check ever fired, and it beat a non-LLM weakest-link repair \cite{demelo2013} on 26 of 40 tasks versus 10 ($p=0.011$). With a full view, a strong model performs the global check. In the second, each ontology pair is curated by its own agent call that sees only that pair, and the results are merged, as the systems of Study~VIII work.

\begin{center}
\small
\begin{tabular}{lrrr}
\toprule
Arm (40 conflict tasks) & Hidden conflict left & Mean F1 & Tasks with intervention\\
\midrule
Full view, one agent & 0/40 & 0.803 & 0\\
Split views, no check & 9/40 (22.5\%) & 0.796 & 0\\
Split views, pairwise check & 9/40 (22.5\%) & 0.796 & 0\\
Split views, network check & \textbf{0/40} & 0.774 & 9\\
\bottomrule
\end{tabular}
\normalsize
\end{center}

With the same model and candidates, split views leave hidden collisions on 9 of 40 tasks (95\% CI 12--38\%). The pairwise check never fires, because these collisions never appear inside one pair; the network check removes all nine (McNemar one-sided $p=0.002$) and never fires on the clean tasks. Detection and repair separate, as Proposition~\ref{prop:z2repair} anticipates: the runtime sent each flagged chain to every pair agent on it, each dropped its own link, and 47 correspondences were removed to resolve 9 conflicts, lowering mean F1 from 0.796 to 0.774 although one well-chosen drop per conflict was available. Detection is exact; repair routing is open (Section~\ref{sec:next-evidence}).

\subsection{Study X: keeping a sufficient context as the rules change}\label{sec:sufficiency-study}

Section~\ref{sec:refine} claims that a runtime can find a sufficient interface without searching for the smallest one, keep it while values change, and must recertify it when rules change; Proposition~\ref{prop:sound} predicts that an interface which has stopped being sufficient either accepts invalid actions or blocks valid ones. Like Study~VII, this study is deterministic, model-free, and frozen before scoring. Each of 200 instances is a project with a shared purchasing budget: spent amount, cap, ordinary and special approval, vendor clearance, and three to five distractor fields. A harness can expose about twelve candidate fields, including three derived ones, each costing 1 to 4 as a proxy for context size. Under v1 a purchase needs ordinary approval and enough budget; from step 12, v2 adds special approval for hazmat purchases; from step 24, v3 adds vendor clearance above a threshold; in 50 instances, v4 from step 32 adds a procurement freeze that no field reports. Each instance runs 40 steps with one random value change per step, and every action is checked at every step: 40{,}000 action-steps, 11{,}161 valid. Every validator accepts an action only if it is valid in all modeled states that look the same through its fields.

\begin{center}
\small
\begin{tabular}{lrrr}
\toprule
Interface & Field cost (share of all fields) & Invalid accepts & Valid actions blocked\\
\midrule
All fields & 100\% & 0 & 0\\
Fixed v1 fields, rules not updated & 19.5\% & 3{,}961 & 0\\
Fixed v1 fields, rules updated & 19.5\% & 0 & 2{,}618\\
Recertified, values not refreshed & 28.0\% & 5{,}244 & 5{,}580\\
Invalidate, fall back to all fields & 75.9\% & 0 & 0\\
\textbf{Certificate loop} & \textbf{28.0\%} & \textbf{0} & \textbf{0}\\
Greedy Set Cover (reference) & 33.6\% & 0 & 0\\
Exact minimum (reference) & 25.5\% & 0 & 0\\
\bottomrule
\end{tabular}
\normalsize
\end{center}

Blocked actions are counted only where exposing every field would have allowed them; the 169 valid actions under the freeze rule that no field can certify are excluded.

\paragraph{What fails.} A context that was sufficient fails in one of two ways once the rules change. Keeping its old rules accepts invalid actions; updating the rules but not the fields blocks valid ones. Each happened in 185 of the 200 instances, and every rule change added a needed field. These are the two faces of Proposition~\ref{prop:sound}: an insufficient context is either unsafe or blocking. Refreshing values and recertifying are different jobs. Recertifying without re-reading values failed both ways, while the 7{,}800 value changes never required a new certificate and each of the 450 rule changes required one.

\paragraph{What the loop achieves.} Sufficiency can be found and kept without solving the NP-hard minimum. The loop accepted nothing invalid and blocked nothing that exposing every field would have allowed, and an independent checker confirmed all 650 of its certificates. On the 600 certified rule versions it chose 3.5 of about 12 fields, matched the exact minimum on 357, and cost 1.11 times the minimum on average, against 1.37 for greedy Set Cover. A certification took 6.4 sufficiency checks and 2.6\,ms at the median; starting from the previous context, each rule change took one refinement round instead of 4.6. Under the freeze rule the loop reported \emph{unresolved} on all 50 instances instead of certifying, and no context can safely approve a purchase there: only a new observation can help. The rule for a harness builder is therefore simple: certify against the rules in force, recertify from the old context when they change, and treat \emph{unresolved} as a request for new information, never as a pass.

\paragraph{Limits.} The 28.0\% includes the freeze steps, where the frozen loop kept 9.2 fields, about half of them distractors; over certified versions only, a breakdown made after scoring gives 25.4\% for the loop and 22.7\% for the exact minimum. The model is small enough to enumerate, so the loop's zeros are guaranteed by its exact checker; the informative numbers are cost, gap to the minimum, and effort. The rules and their changes are part of the model, field cost is a proxy for context size, and agents' proposals are not tested; that is the next experiment.

\subsection{Falsifiable predictions}\label{sec:predictions}

A theory of coherence earns trust by predicting when extra structure should \emph{not} help. Predictions 1--3 are Corollary~\ref{cor:layers} with the corresponding map absent.
\begin{enumerate}
\item \textbf{No distributed state, little value from descent.} One actor with one authoritative state and locally validated actions gains little from local-to-global machinery.
\item \textbf{No representation multiplicity, little value from $\mathcal G$.} One canonical encoding gains nothing from translation machinery; several reversible views need explicit loop checks (Studies VII--IX, including both sides in IX).
\item \textbf{No revisions, little value from $D$.} A one-shot task with immutable inputs gains little from provenance and invalidation.
\item \textbf{A conventional mechanism that owns the state matches the runtime.} Dependency closure (III), Gauss--Seidel and version restart (VI), escrow for an additive budget (IV), and current-state checks where the change breaks a precondition (V) all tie.
\item \textbf{Role decomposition without state ownership can fail.} More planners, verifiers, or messages do not fix invariants no component enforces at commit (Studies I, IV).
\item \textbf{Guards remove targeted failures without a stronger model.} With model and budget fixed, an owning guard removes its invariant's failures specifically (IV: 5/5 to 0/5; V: 24/25 to 0/25), with no false rejections on controls.
\end{enumerate}
Every prediction tested so far has held, including the ones in which more structure should do nothing. One lesson came from a failed panel rather than a prediction: a guard on a nearby but non-identical variable leaves the scored failure intact (Study IV, Panel 1). A guard protects only the state it owns.

\subsection{Scope and next experiments}\label{sec:next-evidence}

The agent panels are modest: five seeds per arm on three TeamBench tasks, 12 injected $\tau^2$ tasks, and 40 tasks per arm in the aliasing and split-view studies. Studies IV, V, and IX use one model; the hidden-fact result of Study I also holds on two models from a second provider. The aliasing tasks are constructed and the $\tau^2$ failures are injected, so Studies I and V measure what happens when these failures occur; Study VIII measures how often one of them occurs in public data. The next experiments extend each of these.
\begin{itemize}
\item \textbf{Natural rates on $\tau^2$-bench.} Study~V measures what the runtime does when stale reads and silent reverts occur, not how often they occur; uninjected frequency on a larger model panel is next.
\item \textbf{Broader $\mathcal G$ evidence.} Repair routing (single arbiter, weakest link, send-to-all) so that the runtime keeps accuracy as well as coherence, and a metamorphic suite of semantics-preserving transforms, since Study IX found the agent less stable under renumbering and reformatting than under resampling.
\item \textbf{Team design from state structure.} Hold model and budget fixed and compare a common-sense design, the design of Section~\ref{sec:team-design}, and a sequential baseline; the prediction is that the derived design matches the sequential one on violations and beats the common-sense one where it leaves a constraint unowned.
\item \textbf{Sufficiency upkeep with agents.} LLM agents plan from the certified context and the runtime validates their proposals, separating what an agent needs to propose useful actions from what the validator needs to stay exact, and testing the per-proposal certificate of Section~\ref{sec:refine}.
\end{itemize}

\section{Relation to prior work}\label{sec:prior}

\paragraph{Agents and benchmarks.} ReAct and Toolformer established the pattern of interleaving model reasoning with tools \cite{yao2023react,schick2023toolformer}, and benchmarks such as WebArena and SWE-bench moved evaluation into persistent interactive environments \cite{webarena2023,swebench2023}. TeamBench enforces role separation \cite{teambench2026}, $\tau^2$-bench distributes control between user and assistant \cite{tau2bench2025}, and BixBench3 evaluates long computational-science trajectories \cite{bixbench3}; these are the environments in which distributed state becomes visible. Failure taxonomies of multi-agent systems report inter-agent misalignment among the most common failures \cite{cemri2025mast}, and practitioners pursue reliability through context engineering, memory, and harnesses. This paper gives those failures an exact condition and those practices a specification.

\paragraph{Mathematical ingredients.} Acyclic database schemes characterize when local relational consistency implies global consistency \cite{beeri1983acyclic}; bounded-width constraint systems give a tractability parameter \cite{freuder1990}; optimistic concurrency control gives serializable commit \cite{kung1981}; Myhill--Nerode gives minimal state for future behavior \cite{hopcroft2006}; group synchronization relates global frames to cycle consistency \cite{lerman2022cycle}; counterexample-guided abstraction refinement comes from model checking \cite{clarke2000cegar}; and Set Cover and syndrome decoding supply the hardness of minimum instrumentation and repair \cite{karp1972,berlekamp1978}. Coordination avoidance and escrow \cite{bailis2014,oneil1986} and team decision theory \cite{hochu1972} address related design questions (Appendices~\ref{app:bridge} and~\ref{app:escapes}). In data integration, identity constraints in linked data and inconsistency in connected alignments \cite{demelo2013,kutz2010whispers} study instances of the representation obstruction that Studies VIII and IX measure.

\paragraph{Compositional software engineering.} Goguen gave concurrent objects a sheaf semantics \cite{goguen1992sheaf}, and view merging computes the integration of partial perspectives as a colimit \cite{sabetzadeh2006}. Lenses synchronize pairs of views \cite{foster2007lenses}, and multi-model consistency management observes that not every multi-ary relation factors into binary ones \cite{stunkel2021comprehensive}, the pairwise-invisible obstruction of Studies VIII and IX. Read against $\mathfrak X$, sheaf semantics, view merging, and multi-model consistency sit in $\mathcal F$, and lenses in $\mathcal G$. None of this work asks whether the views contain enough information for any valid action, which minimal state must be carried forward, or when a stochastic proposer may commit. This paper answers all three.

\paragraph{Knowledge and decentralized decisions.} Coordinated action requires common knowledge, unattainable without guaranteed communication \cite{halpernmoses1990}; FLP and CAP bound consensus and replicated state \cite{flp1985,gilbertlynch2002}; team decision theory makes the information structure itself the object of study \cite{marschakradner1972,hochu1972}; and decentralized decision problems are NP-complete in simple static forms and NEXP-complete as decentralized POMDPs \cite{tsitsiklisathans1985,bernstein2002decpomdp}. The aliasing theorem is the single-decision, validity-valued core of this tradition: it asks whether any policy is guaranteed valid, which yields the exact minimax value, the semantic-bit bound, and the instrumentation hardness proved here. Category theory supplies the language of compositional transformations and effects \cite{fongspivak2019,moggi1991}, groupoids formalize equivalence beyond one group action \cite{weinstein1996}, sheaves integrate compatible local information \cite{robinson2017}, and stacks bridge locality, equivalence, and descent \cite{moerdijkpronk1997,stacks02ZH}. Each of these lines answers part of the problem for one kind of system; none was written for agents.

\section{Conclusion}

When every agent passes its local checks and the team is still wrong, the cause may be distributed state rather than model capability: the information that decides correctness is spread across contexts, participants, representations, and time, and no one holds it all. The theory says what to keep, what it costs, and how to keep it. If the deciding distinction is not in context, no reasoning can guarantee the right action, and we know how well randomization can do and how many bits close the gap. The smallest sufficient context is NP-hard to find, but a sufficient one can be found and recertified by refinement. Pairwise checks suffice on join trees and not in general, translation loops can be certified in linear time, repair can be much harder than detection, and state changes, unlike re-expressions, must be checked at commit against a memory that keeps exactly the history that changes what can happen next. Sheaves, groupoids, categories, and the history state are that sufficiency condition preserved as state is split, translated, and changed.

The ten studies land on both sides of the boundary the theory draws. More reasoning, roles, and samples do not lift hidden-event tasks significantly above chance, while one sentence of context solves them; shown the live count, teams still overspend a shared budget until something owns it; split views produce identity collisions that only a network-level check finds; current-state checks miss stale reads and silent reverts; refinement keeps a sufficient context through rule changes at 28\% of the cost of sharing everything; and conventional mechanisms tie the harness wherever they already own the state. Natural rates are the next question, and the theory says where to look: shared invariants across roles, split views over one object, and state that changes while agents reason about it.

As agent systems move from single assistants to teams, views are split by design and these conditions become common. When a team of agents fails, the fix may be neither a bigger model nor a longer context nor more review, but a harness that owns the state no one holds and puts the right distinctions into each agent's context. No model and no context, however good, removes the need for an owner of a genuinely shared constraint; and a harness helps only with the state it owns, at the scope it checks, over the history it remembers. \textbf{Local intelligence cannot substitute for missing global state, but stochastic local intelligence can act reliably inside an explicit shared state with a governed commit.} Multi-agent AI does not need only bigger models. It also needs context engineering with a target and a harness with a specification, and this paper supplies both. A companion paper \cite{heng2026symmetry} takes up the next question: a team that knows enough can still fail to decide, and it needs symmetry breaking and ownership.
\appendix

\section{Proof conventions and provenance of results}\label{app:provenance}

The appendices contain three kinds of statements.
\begin{enumerate}
\item \textbf{Formulations and results for the present model}, proved in full: the valid-action form of the Observation-Aliasing Impossibility Theorem with its minimax value (linear-programming duality), randomized and adaptive-query extensions (a decision-tree argument), and semantic-bit bounds; exact semantic abstraction; the minimum-instrumentation reduction; certified refinement, which is counterexample-guided abstraction refinement \cite{clarke2000cegar} applied to validity; the prestack form of coherence preservation with its shared-budget counterexample; and the bridging results of Section~\ref{sec:bridge}. The underlying indistinguishability ideas have classical precedents; the contribution is the formulation and its integration into global-coherence semantics.
\item \textbf{Specializations of classical results}: cycle synchronization, Myhill--Nerode minimization, acyclic relational gluing, possibilistic contextuality, treewidth-based CSP algorithms and lower bounds, Max-Cut hardness of signed-graph balancing, Helly-type bounds, invariant confluence \cite{bailis2014}, and optimistic concurrency. We state the form used and cite the source.
\item \textbf{Interpretive projections}: the sheaf, descent, and categorical language that organizes the results above into one framework.
\end{enumerate}
A \emph{guarantee} is worst-case unless a distribution is stated. A policy may use arbitrary internal computation and randomness; the impossibility statements concern missing information, not bounded intelligence.

\section{Observation aliasing, sufficient state, and instrumentation}\label{app:aliasing}
\subsection{Full Observation-Aliasing Impossibility and information-budget results}

This subsection proves the results of Section~\ref{sec:observation} and states the refinements summarized there.

\subsubsection{Deterministic action selection}

Let $X$ be a state space, $Y$ an observation space, $\Omega:X\to Y$, $\mathcal A$ an action space, and $C(x)\subseteq\mathcal A$ the (possibly empty) set of valid actions at $x$. A deterministic policy is a function $\pi:Y\to\mathcal A$.

\begin{proof}[Full proof of Theorem~\ref{thm:aliasing}]
Assume first that a universally valid policy $\pi$ exists. Fix any $y\in\Omega(X)$. For every $x\in\Omega^{-1}(y)$,
\[
\pi(y)=\pi(\Omega(x))\in C(x).
\]
Thus the same element $\pi(y)$ belongs to every validity set in the fiber, so
\[
\pi(y)\in\bigcap_{x\in\Omega^{-1}(y)}C(x),
\]
and the intersection is nonempty.

Conversely, suppose every nonempty observation fiber has nonempty common valid-action intersection. For each $y\in\Omega(X)$ choose one element
\[
a_y\in\bigcap_{x\in\Omega^{-1}(y)}C(x).
\]
Define $\pi(y)=a_y$ on $\Omega(X)$ and arbitrarily elsewhere. For any $x\in X$, the chosen action $a_{\Omega(x)}$ belongs to $C(x)$ by construction. Hence $\pi(\Omega(x))\in C(x)$ for every state. This proves necessity and sufficiency.
\end{proof}

The obstruction is behavioral, not about labels: if a conservative action is valid in every aliased world, there is no impossibility.

\subsubsection{Multi-agent communication without new observation}

Consider $m$ agents with initial local observations $o_i=\Omega_i(x)$. A deterministic $T$-round protocol consists of message functions
\[
m_{i,t}=F_{i,t}\bigl(o_i,\{m_{j,s}:s<t\}\bigr)
\]
and a final decision function of the transcript. Let the joint initial observation be
\[
\Omega(x)=(\Omega_1(x),\ldots,\Omega_m(x)).
\]

\begin{proof}[Proof of Corollary~\ref{cor:roles}]
We prove by induction on the round that every message is a deterministic function of the joint initial observation. At round $1$, each message is a function of the sender's initial observation and therefore of the joint tuple. Assume all messages before round $t$ are functions of the joint initial observation. Then each round-$t$ message is a function of one initial observation and earlier messages, hence again a function of the joint tuple. By induction, the complete transcript is a function of $\Omega(x)$, and so is any deterministic final action computed from that transcript.

Therefore if $\Omega(x)=\Omega(x')$, the protocol produces identical transcripts and identical final actions in the two worlds. If no common valid action exists for the aliased worlds, the protocol cannot be valid on both. The number of roles, rounds, tokens, or internal computations does not change the information partition unless some operation introduces a new observation of the external state.
\end{proof}

Messages may be arbitrarily rich (proofs, plans, code, copies of local observations); none can distinguish states the joint observation identifies.

\subsubsection{Randomized policies}

A randomized policy on observation $y$ is a probability distribution $P_y$ over actions. Its success probability at state $x$ is
\[p_x=P_{\Omega(x)}(C(x)).
\]

\begin{lemma}[Aliased disjoint-action bound]\label{lem:disjoint}
Let $x_1,\ldots,x_s$ share one observation $y$ and let the sets $C(x_i)$ be pairwise disjoint. Then for every randomized policy,
\[
\sum_{i=1}^s p_{x_i}\le 1.
\]
Consequently, $\min_i p_{x_i}\le 1/s$ and the uniform average success is at most $1/s$.
\end{lemma}

\begin{proof}
All states induce the same output distribution $P_y$. Pairwise disjointness gives
\[
\sum_{i=1}^s p_{x_i}
=\sum_{i=1}^sP_y(C(x_i))
=P_y\!\left(\bigsqcup_{i=1}^s C(x_i)\right)
\le1.
\]
The average is therefore at most $1/s$, and at least one term cannot exceed the average.
\end{proof}

A randomized multi-agent protocol with no new external observations can be regarded as a deterministic function of the initial joint observation plus an internal random seed. Conditional on any seed, aliased states still induce the same transcript. After averaging over seeds they therefore induce the same transcript and output distributions. Randomness changes which common distribution is used; it does not refine the observational partition.

\begin{proof}[Full proof of Theorem~\ref{thm:randombits}]
Apply Lemma~\ref{lem:disjoint} with $s=k$. The $k$ states share one output distribution because they share one observation, and their valid-action sets are pairwise disjoint. Hence the sum of their success probabilities is at most one, so at least one state has success probability at most $1/k$. The same output-distribution premise holds for randomized multi-agent protocols by conditioning on the internal random seed as described above.
\end{proof}

\subsubsection{Minimax value and fractional transversals}

\begin{theorem}[Minimax value of an aliased fiber]\label{thm:minimax}
Let $F=\Omega^{-1}(y)$ be a finite fiber and $\mathcal A$ finite. Define the fractional transversal number
\[
\tau^*(F)=\min\Bigl\{\textstyle\sum_{a\in\mathcal A}w_a:\ w\ge0,\ \sum_{a\in C(x)}w_a\ge1\ \ \forall x\in F\Bigr\},
\]
with $\tau^*(F)=\infty$ if some $C(x)$ is empty. The best worst-case success probability of any randomized observation-only policy on $F$ is
\[
v(F)=\max_{P\in\Delta(\mathcal A)}\ \min_{x\in F}\ P(C(x))=\frac1{\tau^*(F)}.
\]
By linear-programming duality, $\tau^*(F)=\max\{\sum_{x}\mu_x:\mu\ge0,\ \sum_{x\in F:\,a\in C(x)}\mu_x\le1\ \forall a\}$. When $\tau^*(F)<\infty$, the normalized dual optimum is a prior on $F$ under which no action succeeds with probability above $v(F)$.
\end{theorem}

\begin{proof}
If some $C(x)$ with $x\in F$ is empty, every policy fails at $x$, so $v(F)=0=1/\tau^*(F)$. Otherwise $v(F)$ is the optimum of the linear program
\[
\max\ t\quad\text{s.t.}\quad\sum_{a\in C(x)}P_a\ge t\ \ (x\in F),\qquad\sum_aP_a=1,\qquad P\ge0.
\]
The uniform distribution is feasible with $t\ge1/|\mathcal A|>0$, so the optimum is positive. If $(P,t)$ is feasible with $t>0$, then $w=P/t$ satisfies every covering constraint and has total weight $1/t$. Conversely, if $w$ is a fractional transversal, then $P=w/\sum_aw_a$ and $t=1/\sum_aw_a$ are feasible. Hence $\max t=1/\min\sum_aw_a=1/\tau^*(F)$.

Write $M_{x,a}=\mathbf 1[a\in C(x)]$. The covering program $\min\{\mathbf 1^\top w:Mw\ge\mathbf 1,\ w\ge0\}$ and the packing program $\max\{\mathbf 1^\top\mu:M^\top\mu\le\mathbf 1,\ \mu\ge0\}$ are both feasible and bounded, so strong duality makes their optima equal. Let $\mu^*$ be a packing optimum and $\pi=\mu^*/\tau^*(F)$. For every action $a$,
\[
\pi\bigl(\{x\in F:a\in C(x)\}\bigr)=\frac{1}{\tau^*(F)}\sum_{x:\,a\in C(x)}\mu^*_x\le\frac1{\tau^*(F)},
\]
so under the prior $\pi$ every policy, randomized or not, succeeds with probability at most $1/\tau^*(F)$.

For the consequences: $v(F)=1$ iff some $P$ has $P(C(x))=1$ for every $x\in F$, iff the support of $P$ lies in $\bigcap_{x\in F}C(x)$, iff that finite intersection is nonempty. If the $C(x)$ are nonempty and pairwise disjoint, any fractional transversal places weight at least one on each of $|F|$ disjoint sets, so $\tau^*(F)\ge|F|$, and one unit on one action per set attains it.
\end{proof}

Three consequences follow. First, $v(F)=1$ exactly when the fiber has a common valid action, so Theorem~\ref{thm:aliasing} is the case $\tau^*=1$. Second, nonempty pairwise-disjoint valid sets give $\tau^*(F)=|F|$, so the bound of Theorem~\ref{thm:randombits} is attained by choosing uniformly among one valid action per state. Third, randomization can help substantially when valid sets overlap without a common element. With $C(x_1)=\{a,b\}$, $C(x_2)=\{b,c\}$, $C(x_3)=\{a,c\}$, every deterministic policy fails in some state, while the uniform distribution on $\{a,b,c\}$ succeeds with probability $2/3=1/\tau^*$ in every state. Randomness does not create information, but it can hedge across an aliased fiber, and $\tau^*$ measures exactly how well.

The integral transversal number $\tau(F)$, the fewest actions that together hit every $C(x)$, satisfies $\tau^*(F)\le\tau(F)$. A deterministic policy has worst-case success $1$ when $\tau(F)=1$ and $0$ otherwise, so the gap between the two is the value of randomization on that fiber.

\subsubsection{Additional semantic bits}

Suppose the $k$ states in Theorem~\ref{thm:randombits} share the same base observation but the runtime may send a deterministic semantic label $z=\phi(x)$ chosen from at most $M=2^b$ labels. Partition the states into label classes $F_1,\ldots,F_q$ with $q\le M$. Within a class, the policy receives identical information and therefore uses one output distribution.

\begin{corollary}[Semantic-information bounds]\label{cor:semanticbits}
Under the setting of Theorem~\ref{thm:randombits}, suppose the runtime provides a deterministic semantic label with at most $M=2^b$ values, refining the $k$ aliased worlds into at most $M$ classes. Then
\[
\text{average success}\le\min\!\left(1,\frac{2^b}{k}\right),\qquad
\text{worst-case success}\le\frac{1}{\lceil k/2^b\rceil},
\]
where the average is under the uniform distribution on the $k$ worlds, and both bounds are attained. Consequently, any worst-case success guarantee above $1/2$ requires $2^b\ge k$, that is, $b\ge\lceil\log_2k\rceil$ for integer $b$: short of exact discrimination, a lossy interface cannot guarantee better than a coin flip on its worst state.
\end{corollary}

\begin{proof}
For each nonempty label class $F_j$, Lemma~\ref{lem:disjoint} implies
\[
\sum_{x_i\in F_j}p_{x_i}\le1\qquad\text{and}\qquad\min_{x_i\in F_j}p_{x_i}\le\frac1{|F_j|}.
\]
Summing the first inequality over all classes gives $\sum_{i=1}^kp_{x_i}\le q\le2^b$, so uniform average success is at most $2^b/k$, and at most one. For the worst case, the $k$ states occupy at most $M$ classes, so by pigeonhole some class has $|F_j|\ge\lceil k/M\rceil$, and the second inequality gives worst-case success at most $1/\lceil k/M\rceil$.

For attainment, split the $k$ states into $\min(M,k)$ classes whose sizes differ by at most one, so every class has at most $\lceil k/M\rceil$ states. In class $F_j$ let the policy choose uniformly among one valid action per state. Each state then succeeds with probability $1/|F_j|\ge1/\lceil k/M\rceil$, and each class contributes total success exactly one, so average success is $\min(M,k)/k$.

Finally, if $M<k$ then $\lceil k/M\rceil\ge2$ and worst-case success is at most $1/2$. A worst-case guarantee above $1/2$ therefore requires $M\ge k$, that is, $b\ge\log_2k$, or $b\ge\lceil\log_2k\rceil$ when $b$ counts whole bits.
\end{proof}

The requirement is targeted: the shared runtime must expose enough semantic distinctions to separate cases that demand different behavior, rather than the whole world. If four otherwise identical program states differ only in which of four mutually exclusive approvals is authoritative, at least two bits of authoritative state are needed to distinguish all four exactly, and with only one bit the worst case is a coin flip. Inference compute and semantic state are different resources, and one cannot pay for the other.

This is a sharp bound for an explicit family in which correct actions are pairwise disjoint. Its role is to demonstrate that inference compute and semantic information are distinct resources in the worst case. For overlapping valid sets, the same argument applies class by class with $1/\tau^*(F_j)$ in place of $1/|F_j|$.

\subsubsection{Adaptive queries}

A $d$-query deterministic protocol is a decision tree of depth at most $d$. Each internal node is labeled by a query $q\in Q$, chosen as a function of $\Omega(x)$ and the answers seen so far; each outgoing edge is labeled by an answer; each leaf is labeled by an action. The \emph{information set} of a node reached from $x$ is the set of states that agree with $x$ on $\Omega$ and on every query answered along the path.

\begin{theorem}[Adaptive aliasing]\label{thm:adaptive}
Let $Q$ be a finite set of side-effect-free queries $q:X\to Y_q$ available in addition to $\Omega$. For $I\subseteq X$ define
\[
\mathrm{Solv}_0(I)\iff\bigcap_{x\in I}C(x)\neq\varnothing,\qquad
\mathrm{Solv}_d(I)\iff\mathrm{Solv}_{d-1}(I)\ \vee\ \exists q\in Q\ \forall y\in q(I):\ \mathrm{Solv}_{d-1}\bigl(I\cap q^{-1}(y)\bigr).
\]
A deterministic protocol that issues at most $d$ adaptively chosen queries guarantees a valid action in every state if and only if $\mathrm{Solv}_d(\Omega^{-1}(y))$ holds for every $y\in\Omega(X)$. With no bound on the number of queries, such a protocol exists if and only if every block of the joint partition generated by $\Omega$ and all $q\in Q$ has a common valid action.
\end{theorem}

\begin{proof}
We prove by induction on $d$ that, for any information set $I$, a protocol of depth at most $d$ that is valid on every state of $I$ exists iff $\mathrm{Solv}_d(I)$. For $d=0$ the protocol is a single action, which must lie in $\bigcap_{x\in I}C(x)$. For $d\ge1$, a protocol of depth at most $d$ either is a single action, requiring $\mathrm{Solv}_0(I)$ and hence $\mathrm{Solv}_{d-1}(I)$ by monotonicity, or begins with some query $q$ and continues on each answer $y\in q(I)$ with a protocol of depth at most $d-1$ on $I\cap q^{-1}(y)$. By the induction hypothesis, the latter exists iff $\mathrm{Solv}_{d-1}(I\cap q^{-1}(y))$ for every $y\in q(I)$. This is the recursive definition of $\mathrm{Solv}_d(I)$. Applying the claim to $I=\Omega^{-1}(y)$ for each base observation gives the first statement.

For the unbounded case, let $\mathcal P$ be the common refinement of the partitions induced by $\Omega$ and by every $q\in Q$. If every block of $\mathcal P$ has a common valid action, the protocol that issues every query and then plays such an action is valid. Conversely, two states in the same block give the same base observation and the same answer to every query, so any protocol follows the same path on both and plays the same action at the same leaf. That action must therefore be valid on the whole block.
\end{proof}

Query side effects are excluded in the theorem. When a query also spends a shared resource, the protocol must additionally keep the spend valid, which bounds the usable depth $d$.

\subsubsection{Approximate decision quality}

\begin{proposition}[Approximate-optimality impossibility]\label{prop:epsilon}
Let $L(x,a)$ be global loss and $L^*(x)=\inf_a L(x,a)$, assumed finite. For tolerance $\varepsilon\ge0$, define the $\varepsilon$-optimal action set
\[
C_\varepsilon(x)=\{a:L(x,a)\le L^*(x)+\varepsilon\}.
\]
An observation-only deterministic policy can guarantee global regret at most $\varepsilon$ if and only if every observation fiber has a common $\varepsilon$-optimal action:
\[
\forall y\in\Omega(X),\qquad
\bigcap_{x\in\Omega^{-1}(y)}C_\varepsilon(x)\neq\varnothing.
\]
Hence if one aliased fiber has no action that is $\varepsilon$-optimal in every compatible global state, no amount of observation-preserving reasoning can guarantee $\varepsilon$-optimal global behavior.
\end{proposition}

\begin{proof}
A policy has worst-case regret at most $\varepsilon$ exactly when its action lies in $C_\varepsilon(x)$ for every state, so this is Theorem~\ref{thm:aliasing} applied to the validity correspondence $C_\varepsilon$. For $\varepsilon=0$, $C_0(x)$ may be empty when the infimum is not attained, and the theorem then correctly reports that no guarantee exists.
\end{proof}

The extension is conditional: if aliased states share one $\varepsilon$-optimal action, the information deficit does not prevent an $\varepsilon$-optimal guarantee.

\subsection{Exact semantic abstraction, safety, progress, and completeness}\label{app:exact}

Let $\phi:X\to Z$ be the runtime abstraction. A validator is a map
\[
V:Z\times\mathcal A\to\{0,1\}.
\]
It is \emph{safe} if $V(\phi(x),a)=1$ implies $a\in C(x)$. It is \emph{complete} if $a\in C(x)$ implies $V(\phi(x),a)=1$. Exactness means both.

\begin{proof}[Proof of Proposition~\ref{prop:sound}]
Fix $z$ and let $A_V(z)=\{a:V(z,a)=1\}$. If $V$ is safe, then for every $x\in X_z$ and every $a\in A_V(z)$ we have $a\in C(x)$. Hence
\[
A_V(z)\subseteq\bigcap_{x\in X_z}C(x)=\mathcal B(z).
\]
Therefore the guaranteed-valid set is the largest set of actions a safe validator can accept at $z$, and a safe validator that makes progress at $z$ exists iff this set is nonempty.

For exactness, suppose first that $\phi(x)=\phi(x')=z$. If $C(x)\neq C(x')$, choose $a$ in the symmetric difference. Exact validation would require different values for $V(z,a)$ in the two states, impossible. Hence exactness implies identical validity profiles within every fiber.

Conversely, if all states in a fiber share one profile, define $V(z,a)=1$ exactly when $a$ belongs to that common profile. This is well-defined and exact. Thus exact validation is possible iff $\phi$ refines the equivalence relation $x\sim_Cx'\iff C(x)=C(x')$.
\end{proof}

\begin{corollary}[Minimal sufficient state count]
Let
\[
K=\bigl|\{C(x):x\in X\}\bigr|.
\]
Every exact abstraction has at least $K$ reachable labels, and the quotient map $x\mapsto C(x)$ has exactly $K$. Therefore the minimum exact deterministic state count is $K$, and any fixed-length binary code needs at least $\lceil\log_2K\rceil$ bits.
\end{corollary}

\begin{proof}
States with different validity profiles cannot share an exact abstraction label by Proposition~\ref{prop:sound}; hence at least one label per profile is necessary. Mapping to the profile itself attains the bound. The bit bound is the pigeonhole principle.
\end{proof}

\paragraph{Interpretation.}
The quotient is task-relative. If the only operation is deciding whether a single fixed action $a_0$ is valid, the semantic state may collapse to one bit even when the underlying world is enormous. If the runtime must validate arbitrary proposals, it must preserve the full validity profile. Thus ``global state'' in this paper means sufficient semantic state for the intended action interface, not necessarily a lossless copy of the physical world.

\subsection{NP-hardness of minimum exact instrumentation}\label{app:instrument}

Choosing the fewest fields that make validation exact is a covering problem over pairs of states that must be told apart, and it inherits Set Cover's hardness and approximability.

Define the decision problem \textsc{Exact-Instrumentation}: given a finite state set $X$, profiles $C(x)$, candidate observables $f_1,\ldots,f_p$, and integer $k$, is there a subset $J\subseteq[p]$ with $|J|\le k$ such that the joint map $(f_j)_{j\in J}$ never merges two states with different validity profiles?

\begin{theorem}
\textsc{Exact-Instrumentation} is NP-complete for explicitly represented finite instances; the corresponding minimum-cardinality optimization problem is NP-hard.
\end{theorem}

\begin{proof}
Membership in NP follows because a proposed subset $J$ can be checked by comparing the selected observable vectors for every pair of explicitly listed states with different profiles.

For NP-hardness reduce \textsc{Set-Cover}. Let the universe be $U=\{u_1,\ldots,u_m\}$, candidate subsets $S_1,\ldots,S_p$, and cover budget $k$. Construct one reference state $x_0$ and one state $x_u$ for every $u\in U$. Let
\[
C(x_0)=\{a_0\},\qquad C(x_u)=\{a_1\}
\]
for two distinct actions $a_0,a_1$. All states $x_u$ share one profile, so exact validation requires only that each $x_u$ be separated from $x_0$; no pair $x_u,x_v$ needs separation.

For each set $S_j$, define Boolean observable
\[
f_j(x_0)=0,\qquad
f_j(x_u)=\mathbf1[u\in S_j].
\]
A selected family $J$ distinguishes $x_0$ from a particular $x_u$ iff some selected set contains $u$. Therefore the selected observables are exact iff the corresponding subsets cover every universe element. There is an exact instrumentation of size at most $k$ iff there is a set cover of size at most $k$. Since \textsc{Set-Cover} is NP-complete \cite{karp1972}, the decision problem is NP-complete and the optimization problem NP-hard.
\end{proof}

The reduction also shows why the problem is naturally a hitting/set-cover problem on \emph{conflicting state pairs}. Let
\[
P=\{\{x,x'\}:C(x)\neq C(x')\}.
\]
Each observable covers the pairs it distinguishes. Exact instrumentation is precisely the problem of covering $P$. Consequently the greedy Set Cover algorithm, which repeatedly selects the observable separating the most still-unseparated conflicting pairs, achieves an $H(|P|)\le\ln|P|+1$ approximation when the finite conflict universe is explicit. The bound cannot be improved much. In the reduction above the conflicting pairs are exactly $\{x_0,x_u\}$ for $u\in U$, so $|P|=|U|$ and instrumentation size equals cover size. Any $(1-\epsilon)\ln|P|$ approximation for \textsc{Exact-Instrumentation} would therefore give a $(1-\epsilon)\ln|U|$ approximation for Set Cover, which is impossible in polynomial time unless $\mathrm P=\mathrm{NP}$ \cite{dinursteurer2014}. The problem generalizes the Minimum Test Collection problem of Garey and Johnson \cite{gareyjohnson1979}, which is the special case in which all validity profiles are distinct and observables are binary. It is also closely related to finding a minimum decision-relative reduct in rough-set theory, where the conflicting pairs $P$ form the discernibility matrix and the decision attribute plays the role of the validity profile \cite{skowronrauszer1992}. The hardness itself is therefore expected. What the validity form adds is its meaning: the attributes that must be kept are exactly those that separate worlds where different actions are valid.

\subsection{Certified refinement}\label{app:refine}

Fix a finite modeled region $X_q$, an action interface $A_q$, rules giving $C$, and finite candidate observables $f_1,\ldots,f_p$. Write $C_q(x)=C(x)\cap A_q$, $F_J=(f_j)_{j\in J}$, and $F=F_{[p]}$ for the full family. Call $J$ \emph{relatively sufficient} if $F_J(x)=F_J(x')$ implies $C_q(x)=C_q(x')$ or $F(x)=F(x')$. A \emph{witness} against $J$ is a pair violating this condition. When the full family is sufficient ($F(x)=F(x')\Rightarrow C_q(x)=C_q(x')$), relative sufficiency is exact sufficiency in the sense of Proposition~\ref{prop:sound}. The loop of Section~\ref{sec:refine} searches for witnesses against relative sufficiency.

\begin{proof}[Proof of Proposition~\ref{prop:refine}]
\emph{(i) Termination.} A witness $(x,x')$ against $J$ has $F_J(x)=F_J(x')$ and $F(x)\neq F(x')$, so some candidate outside $J$ separates the pair. Each round therefore adds a candidate not already selected, and there are at most $p$ rounds.

\emph{(ii) Output.} The addition phase stops only when no witness exists, so its output is relatively sufficient. Sufficiency is monotone: dropping fields only merges more states, so every subset of an insufficient interface is insufficient. The removal pass keeps a field $c$ only when $J\setminus\{c\}$ is insufficient. Later removals shrink $J$ further, so dropping $c$ from the final interface also gives an insufficient interface, and no single field can be removed. If the full family is sufficient, relative sufficiency is exact sufficiency. If it is not, some pair with different profiles is merged by every candidate, so by Proposition~\ref{prop:sound} no subset of candidates is exact; the loop reports this case as unresolved. Its output then merges two states only if they share a profile or no candidate separates them, so it separates everything the full family separates.

\emph{(iii) Soundness at commit.} The certified validator accepts $a$ at observation $y$ exactly when $a\in C_q(x')$ for every $x'\in X_q$ with $F_J(x')=y$. If the actual state $x$ lies in $X_q$ and $y=F_J(x)$, then $x$ is one of those states, so $a\in C_q(x)$. The atomic check of Proposition~\ref{prop:occ} ensures that $y$ and the certified rules are still current when the action takes effect.
\end{proof}

None of the three parts bounds the cost of the output. An irredundant interface can cost more than the minimum, and Theorem~\ref{thm:instrument} rules out a polynomial-time guarantee of minimality unless $\mathrm P=\mathrm{NP}$. The loop's guarantees are sufficiency, irredundancy, and an explicit unresolved verdict. The witness search here enumerates $X_q$. For symbolic models the same query, two copies of the state constrained to agree on $F_J$ and disagree on validity, is a satisfiability problem, and a timeout must be reported as unknown, never as sufficient.

\section{Realization, acyclicity, and local checking}\label{app:realize}
\subsection{Local-to-global realization and the two reconstruction failures}

Let $R:X(U)\to\Match(\{U_i\})$ be the restriction-to-local-views map. In the set-valued case, two logically independent properties are:
\[
\begin{array}{ll}
\text{separatedness:}&R\text{ is injective},\\
\text{gluing/existence:}&R\text{ is surjective}.
\end{array}
\]
A sheaf satisfies both for each declared cover. This decomposition is useful for agent semantics because failure of injectivity and failure of surjectivity have different operational meanings.

If $R$ is not injective, two distinct global states produce the same complete local family. If those states have incompatible action requirements, Theorem~\ref{thm:aliasing} turns the structural non-uniqueness into an action-selection impossibility. If $R$ is not surjective, a family of local states can pass all declared matching tests yet correspond to no global world at all.

\subsubsection{Three-variable cyclic counterexample}

Let local relation $R_{AB}$ contain assignments satisfying $A=B$, $R_{BC}$ assignments satisfying $B=C$, and $R_{AC}$ assignments satisfying $A\neq C$, with Boolean variables. Each relation projects onto $\{0,1\}$ on its shared variable, so the pairwise projections agree. But a global assignment would imply $A=B=C$ from the first two relations and $A\neq C$ from the third, contradiction. Hence pairwise matching does not imply a global realization on this cyclic scope pattern.

The same example can be written as a $\mathbb Z_2$ edge-label problem: equality is label $0$, inequality is label $1$, and the cycle sum is $0+0+1=1\neq0$. The relational contradiction and nontrivial group holonomy are therefore two descriptions of the same obstruction in this case; general sheaf obstructions are richer than group holonomy.

The failure lives in the presheaf of \emph{relations}, not of assignments. Let $\mathcal R(W)$ be the set of relations on a variable set $W$, with projection as restriction. The family $(R_{AB},R_{BC},R_{AC})$ agrees on all overlaps in $\mathcal R$, yet no relation $R$ on $\{A,B,C\}$ projects onto all three, so $\mathcal R$ fails the gluing condition. By contrast, in the presheaf that assigns to each scope the assignments satisfying the constraints contained in it, a family of committed assignments $s_{AB},s_{BC},s_{AC}$ that agree on overlaps would satisfy $A=B=C$ and $A\neq C$ at once; no such family exists, so there is no matching family to glue. This distinction is the possibilistic sheaf model of contextuality \cite{abramsky2011sheaf}, whose cohomological obstructions are studied in \cite{abramsky2012cohomology}; the connection to acyclic database schemes is made in \cite{abramsky2013relational}.

\subsubsection{A state-level gluing failure: shared resources}

Let scopes $U_1,\ldots,U_m$ be disjoint, let $X(U_i)=\{0,\ldots,B\}$ record participant $i$'s spend, and let $X(U)=\{s\in\{0,\ldots,B\}^m:\sum_is_i\le B\}$ with restriction $s\mapsto s_i$. With no overlaps, every tuple in $\prod_iX(U_i)$ is a matching family, and $\rho:X(U)\to\prod_iX(U_i)$ is the inclusion. It is injective, and for $m\ge2$ and $B\ge1$ it is not surjective: the tuple $(B,B,0,\ldots,0)$ has no preimage. Now add a ledger scope $U_0$ whose local state is the spend vector, with local constraint $\sum_is_i\le B$, overlapping each $U_i$ in the $i$-th spend. A matching family must now agree with a ledger state that satisfies the constraint, so every matching family has exactly one global preimage and $\rho$ becomes bijective. This is the formal content of adding an authoritative ledger.

\subsubsection{Acyclic relational schemes}

The BFMY theorem cited in the main text is classical: acyclicity of the relation-scope hypergraph is equivalent to several desirable local-to-global properties, including forms in which pairwise consistency suffices for global consistency \cite{beeri1981properties,beeri1983acyclic}. We use it as a structural boundary, not as an original theorem. Combined with the cyclic example, it says that topology can determine whether local consistency is already globally decisive.

\subsection{No universal fixed order of local satisfiability checking}\label{app:tseitin}

The previous triangle defeats pairwise checks. The next construction shows that merely increasing a fixed locality order does not solve unrestricted global coherence.

Let $G=(V,E)$ be a connected graph with $n=|V|$. Associate a binary variable $x_e\in\mathbb F_2$ with each edge. For every vertex $v$, impose the parity equation
\[
\bigoplus_{e\ni v}x_e=c_v.
\]
Let $B\in\mathbb F_2^{n\times |E|}$ be the vertex-edge incidence matrix and $c\in\mathbb F_2^n$ the charge vector, so the system is $Bx=c$.

\begin{proposition}[No fixed $k$-local satisfiability test]\label{prop:nok-full}
For every fixed $k$, there exists a finite family of local constraints such that every subfamily of at most $k$ constraints is jointly satisfiable but the complete family is not.
\end{proposition}

\begin{proof}
Choose any connected graph with $n>k$ and choose $c$ with odd total parity. Because each edge touches two vertices, summing all rows of $B$ gives zero. Hence summing all equations yields
\[
0=\bigoplus_{v\in V}c_v=1,
\]
so the full system is inconsistent.

For a connected graph, $\operatorname{rank}(B)=n-1$. The left nullspace is therefore one-dimensional and is generated by the all-ones vector; equivalently, the only nonzero linear dependence among all $n$ vertex rows is their total sum. Any proper subset of rows is linearly independent: a dependence among a proper subset, padded with zero coefficients on omitted rows, would produce a nonzero left-null vector not equal to the all-ones vector.

Take any subfamily of $j<n$ equations. Its $j\times |E|$ coefficient matrix has row rank $j$, so its linear map onto $\mathbb F_2^j$ is surjective. Consequently every right-hand side on those $j$ equations is satisfiable, including the restriction of $c$. Every subfamily of size at most $k<n$ is therefore satisfiable although the full family is not.
\end{proof}

This is a statement about bounded-order joint satisfiability, not about every possible notion of local consistency. Some structured classes admit bounded witness theorems, as Helly's theorem and acyclic relational schemes demonstrate. The conclusion is only that no fixed universal locality order works across unrestricted constraint semantics.

\paragraph{Globality is not the same as computational hardness.}
On an $n$-cycle the Tseitin contradiction can require all $n$ vertex equations before inconsistency appears, yet the graph has constant treewidth and total charge parity can be checked in $O(n)$ time. Thus ``requires a global witness'' does not imply ``computationally hard.'' This distinction is essential when interpreting agent failures.

It also separates subfamily checking from propagation. On the cycle, variables are edges and each vertex constraint is binary, so the primal graph is itself an $n$-cycle of treewidth two. Instances whose structure has treewidth less than $k$ are decided by $k$-consistency \cite{dalmau2002}, so 3-consistency, which derives new binary constraints along paths, detects the contradiction although no small subfamily does. Tseitin systems on bounded-degree expanders defeat propagation as well: they require resolution width $\Omega(n)$ \cite{bensasson2001}, and resolution width corresponds to the level of consistency needed \cite{atserias2008}, so $k$-consistency fails for every $k=o(n)$.

For a fixed finite constraint language the question ``what structure determines the order of globality?'' has a precise answer. By the bounded-width theorem of Barto and Kozik, some fixed level of local consistency decides every instance exactly when, for a core language, its algebra of polymorphisms omits tame-congruence types 1 and 2 \cite{barto2014}. Type 1 is an additional obstruction in the general case. Within the Taylor/tractable regime, which already omits type 1, affine equations over nontrivial finite Abelian groups (type 2) are the remaining canonical obstruction; parity is its simplest representative here.

\section{Representation, width, and linear obstructions}\label{app:cycles}
\subsection{Representation synchronization, cycle rank, and verification complexity}

Pairwise translations are globally consistent exactly when every loop composes to the identity, and one check per non-tree edge suffices to verify this in linear time.

Let $\Gamma=(V,E)$ be an undirected graph with an arbitrary reference orientation on each edge. Let $G$ be a group. Each oriented edge $i\to j$ carries $g_{ij}\in G$, with reverse orientation $g_{ji}=g_{ij}^{-1}$. A synchronized representation is a choice of vertex frames $h_i\in G$ satisfying
\[
g_{ij}=h_i h_j^{-1}
\]
for every oriented edge.

\begin{theorem}[Cycle criterion]
A synchronized representation exists iff the ordered product of edge labels around every oriented cycle is the identity.
\end{theorem}

\begin{proof}
If $g_{ij}=h_i h_j^{-1}$, then along a cycle $i_0\to i_1\to\cdots\to i_\ell=i_0$,
\[
\prod_{r=0}^{\ell-1}g_{i_ri_{r+1}}
=\prod_{r=0}^{\ell-1}h_{i_r}h_{i_{r+1}}^{-1}=e
\]
by telescoping.

Conversely, choose a spanning forest $T$ and one root per component. Set each root frame arbitrarily. Traverse the tree. If tree edge $i\to j$ is traversed away from a vertex whose frame is known, define $h_j=g_{ij}^{-1}h_i$; if the stored orientation is reversed use the inverse edge label. Because a tree has one simple root path to each vertex, this defines each frame without ambiguity and makes every tree edge satisfy the synchronization equation.

For a non-tree edge $i\to j$, combine that edge with the unique tree path from $j$ back to $i$. This is its fundamental cycle. The assumed identity product around that cycle is algebraically equivalent to $g_{ij}=h_i h_j^{-1}$. Hence all non-tree edges also synchronize.
\end{proof}

\begin{corollary}[Spanning-forest verification]
For $c$ connected components, it suffices to assign frames along $|V|-c$ tree edges and verify the $\beta_1=|E|-|V|+c$ non-tree edges. With group composition/inversion/comparison cost $c_G$, verification takes
\[
O((|V|+|E|)c_G)
\]
time and $O(|V|)$ stored frames in the straightforward implementation.
\end{corollary}

This is a genuine algorithmic bound only under the declared operation-cost model. For fixed-dimensional rigid transforms $c_G$ can be treated as constant. For dense $d\times d$ matrix groups, naive multiplication contributes an additional $O(d^3)$ factor.

\subsubsection{Counting coherent labelings}

Assume $G$ is finite. After fixing one root frame in each component, the remaining $|V|-c$ frames can be chosen freely. Each choice determines one coherent edge labeling, and distinct choices with fixed roots produce distinct tree-edge labels. Therefore
\[
N_{\mathrm{coh}}=|G|^{|V|-c}.
\]
All unconstrained edge labelings number $|G|^{|E|}$, so a uniformly random labeling is coherent with probability
\[
\boxed{|G|^{-(|E|-|V|+c)}=|G|^{-\beta_1}.}
\]
For this model the cycle rank counts independent global coherence conditions. For $G=\mathbb Z_2$ this also makes $\beta_1$ loop checks necessary. A loop check is a linear functional on labelings in $\mathbb Z_2^{E}$, and the loops form a cycle space of dimension $\beta_1$. Fewer than $\beta_1$ loops span a proper subspace of it, so some loop $\gamma$ lies outside their span, and some labeling has odd sum on $\gamma$ yet even sum on every checked loop. That labeling is incoherent but passes every chosen check.

\subsubsection{State lower bound for arbitrary relative-transform validation}

Fix a graph and define an action query $(i,j,g)$, for vertices $i,j$ in the same connected component, to be valid exactly when the true relative transform $h_ih_j^{-1}$ equals $g$. Distinct synchronized frame configurations modulo one global right action per component induce distinct relative-transform profiles. There are $|G|^{|V|-c}$ such profiles. Proposition~\ref{prop:sound} therefore implies that an exact validator for arbitrary relative-transform proposals needs at least
\[
(|V|-c)\log_2|G|
\]
bits in any fixed-length encoding, while storing one group element per non-root tree vertex attains the same information count up to encoding overhead. This bound is task-specific: a validator answering only one fixed relative-transform query can use less state.

\subsubsection{Incremental append-only synchronization}

When constraints are inserted over time without deletions, weighted union--find (disjoint-set union with group potentials) can maintain component membership and relative frame potentials. Each inserted edge either joins two components or closes a cycle; in the latter case its implied transform is compared with the stored relative potential. Under standard union-by-rank/path-compression assumptions, the data-structure part costs $O(\alpha(|V|))$ amortized per insertion, multiplied by the cost of the necessary group operations. This is a standard dynamic-connectivity specialization; deletions require different machinery and are not covered by the bound.

\subsection{Constraint width and bounded local-to-global reasoning}\label{app:width}

The cost of exact global reasoning is governed by how many variables must be considered together (treewidth), not by the total number of variables or agents.

Consider a finite-domain CSP with $n$ variables, maximum domain size $d$, and a supplied tree decomposition of primal-graph width $w$. Each bag contains at most $w+1$ variables and therefore at most $d^{w+1}$ local assignments. Standard dynamic programming passes compatibility tables along the decomposition, yielding a running time of the form
\[
O(n\,d^{w+1}\,p(n))
\]
where $p(n)$ captures polynomial bookkeeping and constraint-access costs; under common bounded-arity, constant-time lookup assumptions this is summarized as $O(nd^{w+1})$ \cite{freuder1990}. The crucial qualitative fact is that the exponential cost depends only on interaction width (with domain size as its base), not on the number of variables: exact global consistency can be easy for very large but tree-like systems and difficult for smaller densely entangled systems.

Width is also the right parameter from below. Grohe showed that for a recursively enumerable class of bounded-arity structures, CSP restricted to that class of constraint structures is solvable in polynomial time if and only if the cores of those structures have bounded treewidth, assuming $\mathrm{FPT}\neq\mathrm W[1]$ \cite{grohe2007}. Marx showed that, under the Exponential Time Hypothesis, no algorithm solves CSP in time $f(G)\,\|I\|^{o(\mathrm{tw}(G)/\log\mathrm{tw}(G))}$ for primal graphs $G$ of treewidth $\mathrm{tw}(G)$ \cite{marx2010}. The exponential dependence on width therefore cannot be removed by a cleverer general algorithm. For acyclic joins, Yannakakis-style evaluation runs in time linear in input plus output \cite{yannakakis1981}: global coherence is not global recomputation.

\subsection{Detection versus minimum repair}\label{app:syndrome}

\begin{proof}[Proof of Proposition~\ref{prop:z2repair}]
Detection is Corollary~\ref{cor:span} with $G=\mathbb Z_2$. For repair, the coherent labelings are exactly $\{(h_i+h_j)_{ij}:h\in\mathbb Z_2^V\}$, so the minimum number of label changes is
\[
\min_{h\in\mathbb Z_2^V}\bigl|\{ij\in E:\ g_{ij}\neq h_i+h_j\}\bigr|.
\]
Given a graph $\Gamma$, set $g_{ij}=1$ on every edge. An edge then agrees with $h$ exactly when $h_i\neq h_j$, that is, when the bipartition defined by $h$ cuts it. The minimum repair therefore equals $|E|-\mathrm{maxcut}(\Gamma)$, and deciding whether it is at most $\ell$ is equivalent to deciding whether $\Gamma$ has a cut of size at least $|E|-\ell$. Max-Cut remains NP-complete with unit edge weights (\textsc{Simple Max Cut}, problem ND16 in \cite{gareyjohnson1979}; the weighted version is in \cite{karp1972}), so minimum repair is NP-hard.
\end{proof}

The minimum repair is the frustration index of the signed graph $(\Gamma,g)$; on planar graphs it is polynomial \cite{barahona1982,hadlock1975}. For general linear constraints $Ax=b$ over $\mathbb F_q$, realizability is decided by the syndrome $N^Tb$, where the columns of $N$ span $\ker A^T$, and minimum-weight repair is syndrome decoding, which is NP-hard \cite{berlekamp1978}. Exact detection can therefore be cheap while optimal repair is hard, which is why the runtime validates before commit.

\section{Temporal state and atomic commit}\label{app:nerode}
\subsection{Temporal semantic state and Myhill--Nerode minimization}

A runtime needs to remember only those facts about the past that change which futures are valid; this subsection makes that exact and counts the states for a budget with obligations.

Let $\Sigma$ be an action alphabet and $L\subseteq\Sigma^*$ the language of valid finite trajectories. For histories $h,h'\in\Sigma^*$ define future equivalence
\[
h\equiv_Lh'
\iff
\forall u\in\Sigma^*,\quad hu\in L\iff h'u\in L.
\]
The classical Myhill--Nerode theorem states that a language is regular iff this relation has finitely many classes, and in that case the minimal deterministic automaton has exactly one state per class \cite{hopcroft2006}. In the present interpretation, a temporal fact belongs in authoritative semantic state exactly when forgetting it merges histories with different sets of valid futures.

\subsubsection{Budget plus obligations}

Let the system have integer budget $B\ge0$ and $k$ independent Boolean obligations. The alphabet contains unit action \texttt{spend}, idempotent actions $\texttt{mark}_i$, and \texttt{finish}. A trajectory is valid when spending never exceeds the budget and \texttt{finish} occurs only after all obligations are marked; after successful finish no further action is valid.

The obvious deterministic monitor has states
\[
(r,S),\qquad r\in\{0,\ldots,B\},\quad S\subseteq[k],
\]
plus terminal accept and dead/error states. Hence there are $(B+1)2^k+2$ states.

\begin{theorem}[Exact minimality]
The minimal deterministic monitor for this language has exactly
\[
\boxed{(B+1)2^k+2}
\]
states.
\end{theorem}

\begin{proof}
Every pair $(r,S)$ is reachable by spending $B-r$ units and marking exactly the elements of $S$ in some fixed order. We show distinct reachable states are Nerode-distinguishable.

If $S\neq S'$, choose an obligation in the symmetric difference; without loss take $j\in S\setminus S'$. Construct a suffix that marks every obligation missing from $S$, does not mark $j$, and then finishes. Starting from $(r,S)$ this suffix satisfies every obligation (assuming no extra spend is needed), while from $(r',S')$ obligation $j$ remains unmarked, so the suffix is rejected. The suffix contains no \texttt{spend}, so the budgets $r,r'$ play no role.

If $S=S'$ but $r<r'$, append exactly $r+1$ spends, then mark every obligation missing from $S$ and finish. The run from remaining budget $r$ enters the dead state on the last required spend, while the run from $r'$ survives at least that prefix; if $r'$ may equal $r+1$, the subsequent marks and finish remain valid. Thus the states are distinguishable.

The accepting terminal state is distinguished from every other state, including the dead state, by the empty suffix under the convention that $L$ contains exactly successfully terminated histories, or equivalently by any action that is invalid after termination but may be valid before it. The dead state has no accepting continuation and is therefore distinct from every reachable nondead preterminal state, each of which can mark missing obligations and finish. Hence all $(B+1)2^k+2$ states lie in distinct Nerode classes, matching the explicit monitor upper bound.
\end{proof}

The state-count expression is exponential in $k$, but the factored memory representation is not: storing $r$ requires $\lceil\log_2(B+1)\rceil$ bits and the obligation set requires $k$ bits, plus constant terminal status. State count, storage bits, and update complexity are distinct quantities and should not be conflated.

\subsection{Freshness, atomic commit, and serializability}\label{app:occ}

Each authoritative object carries a version. Transaction $T$ records a read set $R(T)$ with the versions used to compute tentative writes $W(T)$. At commit, validation checks that every recorded version still equals the authoritative version; if so, all writes are installed atomically and their versions advanced.

\begin{proposition}[Commit-order serializability]\label{prop:occ-full}
Assume (i) $R(T)$ contains every authoritative state item on which $T$'s tentative writes semantically depend, (ii) an unchanged version implies the relevant value is unchanged, and (iii) validation and installation are atomic with respect to competing commits. Then the committed execution is equivalent to serial execution in commit order.
\end{proposition}

\begin{proof}
Order successful commits $T_1,\ldots,T_m$ by their atomic commit points. We induct on the prefix length. The empty prefix is trivially equivalent to serial execution. Assume the authoritative state immediately after $T_{j-1}$ equals the state produced by serially executing $T_1,\ldots,T_{j-1}$.

Because $T_j$ commits successfully, immediately before its commit every item in its complete semantic read set has exactly the version, and hence the relevant value, that $T_j$ used when computing its tentative writes. Therefore the tentative writes are exactly those $T_j$ would compute if executed against the current serial-prefix state. Atomic installation produces exactly the next serial state. This establishes the induction step and the result.
\end{proof}

The completeness of the semantic read set is essential. If a write depends on an undeclared object, that object can change without invalidating the transaction, defeating the argument. Likewise, phantom/range dependencies require a representation capable of versioning the relevant predicate or range; the proposition is not a substitute for the richer machinery of database concurrency control \cite{kung1981}.

\subsubsection{Complexity under explicit object versions}

With $O(1)$ version lookup, exact validation scans $|R(T)|$ recorded versions and installation touches $|W(T)|$ explicit writes, giving
\[
O(|R(T)|+|W(T)|)
\]
work in this model. There is also a matching worst-case query lower bound for the validation subproblem when the only way to learn whether a read object changed is to inspect its independent version counter: a deterministic validator that skips some counter cannot distinguish the all-current world from a world in which exactly the skipped counter changed. Thus exact validation requires $\Omega(|R(T)|)$ counter inspections in that explicit-counter model. This lower bound should not be generalized to systems with stronger aggregate freshness summaries.

\section{Layer induction, escapes, and separations}\label{app:layers}
\subsection{Symmetry quotienting: exact semantic reduction and runtime caveat}\label{app:quotient}

Let a finite group $G$ act on a finite candidate space $X$. Suppose feasibility $F$ and objective $J$ are invariant:
\[
F(gx)=F(x),\qquad J(gx)=J(x).
\]
Then both descend to the orbit space $X/G$, and
\[
\min_{x:F(x)}J(x)=\min_{[x]\in X/G:F([x])}J([x]).
\]
If the action is free, every orbit has $|G|$ elements and
\[
|X/G|=|X|/|G|.
\]
For non-free actions, Burnside's lemma gives the exact orbit count.

This is an exact reduction in the number of semantically distinct candidates. It is \emph{not} by itself a wall-clock theorem: a solver must still generate or canonicalize orbit representatives, and equivalence testing may be expensive. A runtime claim of $O(|X|/|G|)$ therefore requires an explicit model in which representative generation/canonicalization is cheap relative to candidate evaluation.

There is also a categorical reason not to quotient too early. For $G$-sets, taking orbits is a left adjoint and therefore preserves colimits, not limits in general. Matching and gluing conditions are limit-like constructions. Objectwise orbit quotients can therefore erase cycle/transition information needed to decide whether local pieces globally realize. Groupoid-valued state retains equivalence morphisms instead of replacing them immediately by coarse orbit labels.

\subsection{Separation countermodels for \texorpdfstring{$H,\mathcal C,\mathcal G,\mathcal F$}{H, C, G, F}, and temporal state}\label{app:separation}

The independence of the operational symbols is best stated through finite countermodels; they form distinct structures rather than a nested hierarchy.

\paragraph{Topology does not determine transformation semantics.}
Take identical two-node topology $u\to v$. System 1 implements $v:=u$; System 2 implements $v:=1-u$. Reachability is identical, but predicted post-state differs. Therefore $H$ alone cannot recover $\mathcal C$.

\paragraph{Transformation semantics does not determine realization.}
Give three Boolean scopes identical local update operations. In one global semantics require $A=B$, $B=C$, and $A=C$; in another require $A=B$, $B=C$, and $A\neq C$. Local operations and scope topology can be identical while one matching family is globally realizable and the other is not. A realization condition adds information beyond local transition typing.

\paragraph{Surface identity does not determine semantic equivalence.}
Two records may have identical schemas and values while denoting two distinct objects, or may have different schemas while denoting two representations of one object. Neither topology nor syntax determines the groupoid of meaning-preserving translations.

\paragraph{Present values do not determine temporal validity.}
Two artifacts can contain identical values while one was derived from current dependency versions and the other from superseded versions. If future validity depends on provenance, an abstraction that retains only present payloads merges states with different validity profiles and is therefore inexact by Proposition~\ref{prop:sound}.

These examples justify distinct semantic responsibilities without asserting set-theoretic containments such as $H\subsetneq H+\mathcal C$. The structures interact, but none of the counterexamples can be repaired by merely renaming one structure as another.

\subsection{From necessity to structure: full proofs}\label{app:bridge}

Section~\ref{sec:bridge} states these results in words. This subsection gives the formal statements and proofs.

\subsubsection{Restriction, translation, and transition}

Let $\rho(x)=(x|_{U_i})_i$ be the family of local views on a declared cover.

\begin{proposition}[Local sufficiency]\label{prop:localsuff}
A validator that reads only the local views $\rho(x)$ can be exact if and only if $\rho(x)=\rho(x')$ implies $C(x)=C(x')$. It accepts only realizable local families if and only if it rejects every matching family outside $\rho(X)$.
\end{proposition}

\begin{proof}
If $\rho(x)=\rho(x')$ implies $C(x)=C(x')$, define $V(\rho(x),a)=1$ exactly when $a\in C(x)$. This is well defined because states with the same local views share one validity profile, and it is exact by construction. Conversely, a validator reading only $\rho(x)$ is a validator for the abstraction $\phi=\rho$, and Proposition~\ref{prop:sound} shows it can be exact only if $\rho$ never merges states with different profiles.

For the second statement, a matching family $m$ is realizable exactly when $m\in\rho(X)$. A validator that accepts some $m\notin\rho(X)$ therefore accepts an unrealizable family, and one that rejects every such $m$ accepts only realizable families. When $\rho$ is a bijection onto the matching families, as for a sheaf on the declared cover, the check is automatic; otherwise it is the realization check whose cost Sections~\ref{sec:coherence-math} and~\ref{sec:architecture} analyze.
\end{proof}

The two conditions are the two sheaf axioms, weakened in one useful way. Separatedness is needed only up to validity: the runtime needs a sheaf of the validity quotient, not of the whole world. Gluing is needed only for families that some participant can actually commit. The topology $H$ declares which scopes exist, and Section~\ref{sec:coherence-math} exhibited both failures: aliasing when the cover is too coarse, and the shared budget of Section~\ref{sec:sharedbudget} when no scope contains a global constraint.

Participants rarely hold the sufficient state in one format. Let $Z$ be the sufficient state, let participant $i$ hold $r_i=e_i(z)$ for an encoding $e_i:Z\to R_i$, and let a translation $t_{ij}:R_i\to R_j$ satisfy $t_{ij}\circ e_i=e_j$.

\begin{proposition}[Translations of a sufficient state form a groupoid]\label{prop:groupoid}
\begin{enumerate}[label=(\roman*)]
\item If $e_j$ is injective, so is $e_i$: no translation can restore a distinction lost by its source encoding.
\item If every $e_i$ is a bijection, then $t_{ij}=e_j\circ e_i^{-1}$ is uniquely determined, $t_{ii}=\mathrm{id}$, $t_{ji}=t_{ij}^{-1}$, and $t_{jk}\circ t_{ij}=t_{ik}$.
\item Conversely, on a connected graph of participants, a family of bijections that composes to the identity around every cycle arises from bijective encodings of one state, unique up to one global re-encoding.
\end{enumerate}
\end{proposition}

\begin{proof}
(i) If $e_i(z)=e_i(z')$, then $e_j(z)=t_{ij}(e_i(z))=t_{ij}(e_i(z'))=e_j(z')$, and injectivity of $e_j$ gives $z=z'$.

(ii) If $e_i$ is a bijection, $t_{ij}\circ e_i=e_j$ forces $t_{ij}=e_j\circ e_i^{-1}$. Then $t_{ii}=e_ie_i^{-1}=\mathrm{id}$, $t_{ji}\circ t_{ij}=e_ie_j^{-1}e_je_i^{-1}=\mathrm{id}$, and $t_{jk}\circ t_{ij}=e_ke_j^{-1}e_je_i^{-1}=e_ke_i^{-1}=t_{ik}$. In particular the composite of translations around any cycle of participants is the identity.

(iii) Let the participants form a connected graph whose oriented edges carry bijections $t_{ij}:R_i\to R_j$, with $t_{ji}=t_{ij}^{-1}$, composing to the identity around every cycle. Fix a root participant $i_0$, set $Z=R_{i_0}$ and $e_{i_0}=\mathrm{id}$, and for each $i$ let $e_i$ be the composite of translations along any path from $i_0$ to $i$. Two paths differ by a cycle whose composite is the identity, so $e_i$ is well defined, and for every edge $t_{ij}\circ e_i=e_j$ because a path to $j$ may pass through $i$. If $e'_i:Z'\to R_i$ is another family of bijections with $t_{ij}\circ e'_i=e'_j$, then along every path $e'_i=e_i\circ\sigma$ with $\sigma=e_{i_0}^{-1}\circ e'_{i_0}$, so the two families differ by one global re-encoding $\sigma:Z'\to Z$. This is the argument of Theorem~\ref{thm:cycles} with the group replaced by the groupoid of bijections between the sets $R_i$.
\end{proof}

Part (i) is the aliasing theorem at the translation map: a lossy translation merges states. Parts (ii) and (iii) say that the translations of a sufficient state are exactly a groupoid with trivial holonomy. When encodings are fixed only up to a symmetry of the object---a coordinate frame, a labeling, a unit choice---translations are determined only up to that symmetry, the groupoid acquires nontrivial automorphism groups, and quotienting by the symmetry before checking cycles discards the holonomy that detects the error (Section~\ref{sec:representation}).

Let each action $a$ act on world states by $T_a:X\to X$, with an absorbing dead state for invalid actions, and let a word $w=a_1\cdots a_n$ act by $T_w=T_{a_n}\circ\cdots\circ T_{a_1}$. A state abstraction $\phi$ is \emph{self-updating} if $\phi(x)=\phi(x')$ implies $\phi(T_ax)=\phi(T_ax')$ for every action $a$. Equivalently, there are update maps $f_a$ with $\phi\circ T_a=f_a\circ\phi$, and these compose: $f_{ab}=f_b\circ f_a$.

\begin{theorem}[Coarsest self-updating sufficient state]\label{thm:selfupdate}
Define $x\approx x'$ if and only if $C(T_wx)=C(T_wx')$ for every finite action word $w$, including the empty word. Then $\approx$ is exact and self-updating, and every exact self-updating abstraction refines $\approx$. The coarsest state that is both exact and updatable without re-reading the world is therefore $X/{\approx}$.
\end{theorem}

\begin{proof}
The relation $\approx$ is an equivalence relation because it is defined by equality of the family $(C(T_wx))_w$.

\emph{Exact.} The empty word gives $C(x)=C(x')$ whenever $x\approx x'$, so $\approx$ refines $\sim_C$ and is exact by Proposition~\ref{prop:sound}.

\emph{Self-updating.} Suppose $x\approx x'$ and fix an action $a$. For every word $w$, the word $aw$ acts by $T_{aw}=T_w\circ T_a$, so $C(T_w(T_ax))=C(T_{aw}x)=C(T_{aw}x')=C(T_w(T_ax'))$. Hence $T_ax\approx T_ax'$.

\emph{Coarsest.} Let $\phi$ be exact and self-updating, and suppose $\phi(x)=\phi(x')$. By induction on the length of $w$, self-updating gives $\phi(T_wx)=\phi(T_wx')$ for every word. Exactness then gives $C(T_wx)=C(T_wx')$ for every word, so $x\approx x'$. Thus $\phi$ refines $\approx$.

\emph{Update maps.} If $\phi$ is self-updating, $f_a(\phi(x))=\phi(T_ax)$ is well defined on the image of $\phi$, and $T_{ab}=T_b\circ T_a$ gives $f_{ab}=f_b\circ f_a$.
\end{proof}

\paragraph{Relation to Myhill--Nerode.}
Let invalid actions send every state to an absorbing dead state with $C(\mathrm{dead})=\varnothing$, fix an initial state $x_0$, and let $L$ be the language of histories every action of which is valid when taken. For valid histories $h,h'$ with $x=T_hx_0$ and $x'=T_{h'}x_0$, we have $x\approx x'$ if and only if $h\equiv_Lh'$. Indeed, if $hw\notin L$ then $T_wx$ is dead; and if $hw\in L$, then $a\in C(T_wx)$ exactly when $hwa\in L$. Equality of all validity sets along all words is therefore equality of all valid continuations. On reachable states, $X/{\approx}$ is the minimal monitor of Appendix~\ref{app:nerode}.

\subsubsection{The components in categorical form}

The refined state model tracks equivalent representations through a groupoid-valued assignment
\[
\mathcal S:(\mathsf U,J)^{op}\rightarrow\mathbf{Gpd}.
\]
Here $\mathsf U$ is the collection of local scopes, $J$ specifies which families of scopes count as covers of a larger scope, $(-)^{op}$ reverses restriction arrows in the usual presheaf convention, and $\mathbf{Gpd}$ is the category of groupoids---categories whose every morphism is invertible. Since a groupoid is a category in which every morphism is invertible, $\mathcal G$ is not a second kind of mathematics beside $\mathcal C$. It is the reversible core of representation change, kept as a separate slot so the runtime can tell a genuine change of state (revising a design from $1\,\mathrm{m}$ to $1.2\,\mathrm{m}$) from a re-encoding ($1\,\mathrm{m}$ as $100\,\mathrm{cm}$). Descent is the existence-and-coherence condition for gluing those local objects and equivalences into a global object. Noninvertible transformations live in $\mathcal C$; formally they are morphisms of a category-valued prestack whose groupoid core is $\mathcal S$, and Proposition~\ref{prop:preserve} shows that, unlike re-representations, they must be checked for realization at commit. Descent defines which candidate states count as coherent; it performs no search itself.

\subsubsection{Prior art for the same pattern}

The bridge of Section~\ref{sec:bridge}---a requirement of sufficient information, then structures that preserve it as it is distributed---recurs elsewhere. In team decision theory, what decentralized decision makers can achieve depends on the information structure \cite{marschakradner1972,witsenhausen1968}, and Ho and Chu's partially nested condition restores tractability \cite{hochu1972}; the shared-budget ledger is an instance. An atlas describes an object without a global chart by local charts, transition maps, the cocycle condition, and gluing \cite{steenrod1951}, which correspond to $H$, $\mathcal G$, cycle consistency, and $\mathcal F$. The validity quotient is a minimal sufficient statistic \cite{bahadur1954}, and a recursively updatable one, such as a filter's belief state, is the self-updating state of Theorem~\ref{thm:selfupdate}, with Kalman's minimal realization as its linear form \cite{hokalman1966,kalman1963}. In version control, a merge conflict is a pushout that fails to exist \cite{mimram2013}: a gluing failure.

\subsection{Other ways out of the impossibility, and what each costs}\label{app:escapes}

Theorem~\ref{thm:aliasing} says the missing distinction must come from somewhere, not where. The alternatives below locate the proposed runtime in a design space, as the CAP theorem locates distributed databases \cite{gilbertlynch2002}.

\begin{center}
\small
\begin{tabular}{p{3.4cm}p{4.6cm}p{5.4cm}}
\toprule
Alternative & What it gives up & When the theory predicts it suffices\\
\midrule
Weaken the goal & progress & a guaranteed-valid action exists and stalling is acceptable\\
Obtain the information elsewhere & nothing formally, but the authoritative fact must still exist & queries or messages can split every aliased fiber within budget\\
Restructure state so local checks suffice & flexibility & the invariant is invariant-confluent, or can be split into local quotas\\
Accept probabilistic success & worst-case guarantees & the information is in the observation and errors are tolerable\\
Conventional specialized mechanisms & generality across boundaries & one mechanism already captures the full relevant state relation\\
\bottomrule
\end{tabular}
\normalsize
\end{center}

\paragraph{Weaken the goal.} If some action is valid throughout an aliased fiber (stop, escalate, do nothing), Theorem~\ref{thm:aliasing} no longer bites, at the price of progress (Proposition~\ref{prop:sound}). Asking for confirmation is this route; it fits when stalling is cheap.

\paragraph{Obtain the information another way.} Agents can query tools, broadcast state, write to a blackboard \cite{erman1980}, or delegate to an agent that sees everything. Theorem~\ref{thm:adaptive} says exactly when bounded queries suffice, and Corollary~\ref{cor:roles} caps communication at the joint observation. Every such route presupposes that the fact exists somewhere in authoritative form: the choice is between agents pulling authoritative state and a runtime pushing or enforcing it, not between having and not having it.

\paragraph{Restructure state so that local checks suffice.} Transactions that are \emph{invariant-confluent} need no global coordination \cite{bailis2014}, and conflict-free replicated data types realize suitable cases \cite{shapiro2011}; in the language of Section~\ref{sec:realization}, the reachable committed families glue. Independent spend increments are not invariant-confluent: from $s=0$, two replicas can commit $s_1=B$ and $s_2=B$, each valid, whose merge has $\sum_is_i=2B>B$. Escrow splits the bound into quotas $B_i$ with $\sum_iB_i=B$ \cite{oneil1986}, shrinking each local state space to $\{0,\ldots,B_i\}$ so that every tuple of local states is realizable, at the cost of flexibility. Study~IV shows that this suffices for an additive budget.

\paragraph{Accept probabilistic success.} When the distinction is present in the observation, learning can close the gap; the theorems bound only what is achievable when it is absent (Theorem~\ref{thm:minimax}). This fits when violations are rare, tolerable, and cheap to detect afterward; Proposition~\ref{prop:z2repair} indicates when afterward is expensive.

\paragraph{Conventional specialized mechanisms.} Transactions, dependency engines, version-aware restarts, residual solvers, and workflow engines implement parts of the semantics and tie with the general runtime where they capture the full relevant state (Studies III and VI, Corollary~\ref{cor:layers}). The general runtime is warranted where no single mechanism spans the state, such as a budget shared across roles with different tools, or provenance across artifacts managed by different systems.

\subsection{Full descent-style coherence-preservation argument}\label{app:preserve}

This states Section~\ref{sec:transitions} formally: re-expressions preserve coherence, and state changes must be checked at commit. The groupoid core isolates the invertible morphisms ($\mathcal G$), while $\mathcal C$ carries the generally noninvertible state changes, so the state assignment for transitions is category-valued with groupoid core $\mathcal S$. For each scope $U$ let $\mathcal S_{\mathcal C}(U)$ be a category of local states and local actions, with restriction functors, whose groupoid core is $\mathcal S$. Assume $\mathcal S_{\mathcal C}$ is a \emph{prestack}: local morphisms glue uniquely, so the comparison functor from $\mathcal S_{\mathcal C}(U)$ to descent data on a cover is fully faithful. A descent datum is \emph{effective} when it lies in the essential image of that functor, that is, when it comes from a global state; the coherent states are the effective data.

If $\mathcal S_{\mathcal C}$ were a stack, every descent datum would be effective and preservation would be free; the interesting setting is a prestack that need not be a stack.

Fix a cover $\{U_i\to U\}$. A descent datum consists of objects $x_i$ of $\mathcal S_{\mathcal C}(U_i)$ and overlap isomorphisms
\[
\phi_{ij}:x_i|_{U_{ij}}\overset{\sim}{\longrightarrow}x_j|_{U_{ij}}
\]
satisfying the cocycle condition on triple overlaps,
\[
\phi_{jk}\circ\phi_{ij}=\phi_{ik}.
\]
A morphism of descent data $(x_i,\phi_{ij})\to(y_i,\psi_{ij})$ is a family $f_i:x_i\to y_i$ with $\psi_{ij}\circ f_i|_{U_{ij}}=f_j|_{U_{ij}}\circ\phi_{ij}$. Restriction defines a comparison functor $\kappa:\mathcal S_{\mathcal C}(U)\to\mathrm{DD}(\{U_i\})$. The assignment is a \emph{prestack} for the cover when $\kappa$ is fully faithful, equivalently when morphisms glue uniquely, and a \emph{stack} when $\kappa$ is moreover essentially surjective. A datum is \emph{effective} when it is isomorphic to $\kappa(x)$ for some $x$.

With $\mathcal S_{\mathcal C}$ a prestack, let $(x_i,\phi_{ij})$ be effective with global realization $x$. Suppose local actions $f_i:x_i\to y_i$, together with target overlap isomorphisms $\psi_{ij}$ satisfying the cocycle law, form a morphism of descent data:
\[
\psi_{ij}\circ f_i=f_j\circ\phi_{ij}.
\]

\begin{proposition}[Coherence preservation and its limit]\label{prop:preserve}
\begin{enumerate}[label=(\roman*)]
\item If every $f_i$ is an isomorphism, the target datum $(y_i,\psi_{ij})$ is effective and the family glues to a unique global isomorphism $f:x\to y$.
\item If the target datum is effective with realization $y$, the family glues to a unique global transition $f:x\to y$.
\item For non-invertible $f_i$ the target datum need not be effective: there are prestacks, effective data, and locally valid actions whose combined result has no global realization.
\end{enumerate}
\end{proposition}

\begin{proof}
Assume $\mathcal S_{\mathcal C}$ is a prestack, $(x_i,\phi_{ij})\cong\kappa(x)$, and $(f_i)$ is a morphism of descent data to $(y_i,\psi_{ij})$.

(i) If every $f_i$ is an isomorphism, then $(f_i^{-1})$ is also a morphism of descent data, since $\psi_{ij}\circ f_i=f_j\circ\phi_{ij}$ gives $f_j^{-1}\circ\psi_{ij}=\phi_{ij}\circ f_i^{-1}$. Hence $(f_i)$ is an isomorphism in $\mathrm{DD}$ and $(y_i,\psi_{ij})\cong(x_i,\phi_{ij})\cong\kappa(x)$, so the target is effective, realized for instance by $y=x$. Full faithfulness of $\kappa$ gives a unique $f:x\to y$ with $\kappa(f)$ equal to the composite isomorphism, and $f$ is an isomorphism because $\kappa$ reflects isomorphisms.

(ii) If $(y_i,\psi_{ij})\cong\kappa(y)$, then $(f_i)$ determines a morphism $\kappa(x)\to\kappa(y)$ in $\mathrm{DD}$, and full faithfulness gives a unique $f:x\to y$ inducing it.

(iii) Take disjoint scopes $U_1,\ldots,U_m$ with $m\ge2$ and a budget $B\ge1$. Let $\mathcal S_{\mathcal C}(U_i)$ be the poset $\{0,\ldots,B\}$ viewed as a category, with a unique morphism $s\to s'$ when $s\le s'$ (spending never decreases). Let $\mathcal S_{\mathcal C}(U)$ be the poset $\{s\in\{0,\ldots,B\}^m:\sum_is_i\le B\}$ under the componentwise order, with restriction $s\mapsto s_i$. With no overlaps, descent data are arbitrary tuples and morphisms of descent data are componentwise inequalities. A global morphism $s\to t$ exists exactly when $s_i\le t_i$ for every $i$, and it is unique, so $\kappa$ is fully faithful and $\mathcal S_{\mathcal C}$ is a prestack. The datum $x=(0,\ldots,0)$ is effective. The local actions $f_1:0\to B$, $f_2:0\to B$, and identities elsewhere form a morphism of descent data, each $f_i$ is a valid local transition, and the target $(B,B,0,\ldots,0)$ is not effective because $2B>B$.

For the trajectory statement, let $X_{\mathrm{coh}}$ be the class of effective data and assume the initial state is effective. If every committed step is either a family of isomorphisms, covered by (i), or has a target that the runtime has checked to be effective, covered by (ii), then each step maps an effective state to an effective state and glues to a global transition. Induction on $t$ proves that every finite committed prefix remains coherent.
\end{proof}

The counterexample in (iii) is the shared budget of Section~\ref{sec:sharedbudget}. Each role's local transition is valid, every overlap condition holds vacuously, and the combined state is not a global state. A runtime that checks only local validity and overlap squares accepts it; a runtime that checks effectiveness, here the ledger constraint, rejects it. The descent argument describes one before/after family; a runtime that commits only part of it can expose an intermediate state, which is why atomic commit is needed to implement the transition as one step.

\section{Benchmark details}\label{app:empirical-app}
\subsection{Sobieski equation reconstruction}\label{app:sobieski}

We independently reconstructed the documented Propulsion and Mission submodels of the Sobieski SSBJ benchmark and matched published feasible-state values to numerical precision.

\begin{center}
\begin{tabular}{lrrr}
\toprule
Quantity & Reconstruction & GEMSEO reference & Absolute error\\
\midrule
SFC $y_{34}$ & 1.1075457743 & 1.1075457700 & $4.26\times10^{-9}$\\
Engine weight $y_{31}$ & 6354.324306910 & 6354.324306910 & $2.72\times10^{-10}$\\
ESF $y_{32}$ & 0.5027962509 & 0.5027962500 & $8.97\times10^{-10}$\\
Range $y_4$ (nm) & 535.788448111 & 535.788448180 & $6.89\times10^{-8}$\\
\bottomrule
\end{tabular}
\end{center}

This validates the controlled substrate, not the general theory.

\end{document}